\documentclass[a4paper]{article}[12pts]

\usepackage[english]{babel}
\usepackage[utf8x]{inputenc}
\usepackage[T1]{fontenc}

\normalsize

\usepackage[a4paper,top=1in,bottom=1in,left=1in,right=1in,marginparwidth=1in]{geometry}

\usepackage{setspace}
\usepackage{amsmath}
\usepackage{amssymb} 
\usepackage{amsthm}
\usepackage{amsfonts, mathtools}
\usepackage{algorithm,algpseudocode}
\usepackage{bm}
\usepackage{comment}
\usepackage{graphicx}
\usepackage{multirow, booktabs}
\usepackage{caption}
\usepackage{subcaption}
\usepackage[colorinlistoftodos]{todonotes}
\usepackage[colorlinks=true, allcolors=blue]{hyperref}

\newcommand{\E}{\mathbb{E}}

\newcommand{\prob}{\mathbb{P}}

\newtheorem{proposition}{Proposition}
\newtheorem{lemma}{Lemma}

\newtheorem{assumption}{Assumption}

\usepackage{natbib}

\title{The Symmetry between Arms and Knapsacks: A Primal-Dual Approach for Bandits with Knapsacks}
\author{Xiaocheng Li$^\dagger$ \and Chunlin Sun$^\ddagger$ \and  Yinyu Ye$^\Diamond$}
\date{\small 
$^\dagger$ Imperial College Business School\\
$^\ddagger$ Institute for Computational and Mathematical Engineering, Stanford University\\
$^\Diamond$ Department of Management Science and Engineering, Stanford University
}

\begin{document}
\maketitle

\onehalfspacing

\begin{abstract}
In this paper, we study the bandits with knapsacks (BwK) problem and develop a primal-dual based algorithm that achieves a problem-dependent logarithmic regret bound. The BwK problem extends the multi-arm bandit (MAB) problem to model the resource consumption associated with playing each arm, and the existing BwK literature has been mainly focused on deriving asymptotically optimal distribution-free regret bounds. We first study the primal and dual linear programs underlying the BwK problem. From this primal-dual perspective, we discover symmetry between arms and knapsacks, and then propose a new notion of sub-optimality measure for the BwK problem. The sub-optimality measure highlights the important role of knapsacks in determining algorithm regret and inspires the design of our two-phase algorithm. In the first phase, the algorithm identifies the optimal arms and the binding knapsacks, and in the second phase, it exhausts the binding knapsacks via playing the optimal arms through an adaptive procedure. Our regret upper bound involves the proposed sub-optimality measure and it has a logarithmic dependence on length of horizon $T$ and a polynomial dependence on $m$ (the numbers of arms) and $d$ (the number of knapsacks). To the best of our knowledge, this is the first problem-dependent logarithmic regret bound for solving the general BwK problem.
\end{abstract}

\section{Introduction}

The Multi-Armed Bandit (MAB) problem is a problem in which a limited amount of resource must be allocated between competing (alternative) choices in a way that maximizes the expected gain \citep{gittins2011multi}. It is a benchmark problem for decision making under uncertainty that has been studied for nearly a century. As a prototypical reinforcement learning problem, MAB problem exemplifies the exploration–exploitation tradeoff dilemma \citep{weber1992gittins}. The original problem first formulated in its predominant version in \citep{robbins1952some}, has inspired a recent line of research that considers additional constraints that reflect more accurately the reality of the online decision making process. \textit{Bandits with Knapsacks} (BwK) was introduced by \cite{badanidiyuru2013bandits} to allow more general constraints on the decisions across time, in addition to the customary limitation on the time horizon. The BwK problem, as a general framework, encompasses a wide range of applications, including dynamic pricing and revenue management \citep{besbes2012blind}, online advertisement \citep{mehta2005adwords}, network and routing \citep{agrawal2014dynamic}, etc. 

While the existing BwK literature \citep{badanidiyuru2013bandits,agrawal2014bandits} has derived algorithms that achieve optimal problem-independent regret bounds, a problem-dependent bound that captures the optimal performance of an algorithm on a specific BwK problem instance remains an open question. For the setting of standard MAB problem, the problem-dependent bound has been well understood, and an upper bound with logarithmic dependence on $T$ can be achieved by both UCB-based algorithm \citep{auer2002finite} and Thompson sampling-based algorithm \citep{agrawal2012analysis}. In this paper, we focus on developing a problem-dependent bound for the BwK problem and identify parameters that characterize the hardness of a BwK problem instance. 

Two existing works along this line are \citep{flajolet2015logarithmic} and \citep{sankararaman2020advances}. The paper \citep{flajolet2015logarithmic} considers several specific settings for the problem, and for the general BwK problem, it achieves an $O(2^{m+d}\log T)$ regret bound (with other problem-dependent parameters omitted) where $m$ is the number of arms and $d$ is the number of the knapsacks/resource constraints. In addition, the results in \citep{flajolet2015logarithmic} require the knowledge of some parameters of the problem instance a priori. The recent work \citep{sankararaman2020advances} considers the BwK problem under the assumptions that there is only one single knapsack/resource constraint and that there is only one single optimal arm. In contrast to these two pieces of work, we consider the problem in its full generality and do not assume any prior knowledge of the problem instance. We will further compare with their results after we present our regret bound. 

%The problem of Bandits with Knapsacks resembles with the problem of online linear programming, especially the multi-dimensional online linear programming problem \cite{molinaro2013geometry, agrawal2014dynamic, kesselheim2014primal} in that both take a linear programming problem as its underlying form. The key different is that  

Specifically, we adopt a primal-dual perspective to study the BwK problem. Our treatment is new in that
we highlight the effects of resources/knapsacks on regret from the dual perspective and define the sub-optimality measure based on the primal and dual problems jointly. Specifically, we first derive a generic upper bound that works for all BwK algorithms. The upper bound consists of two elements: (i) the number of times for which a sub-optimal arm is played; (ii) the remaining knapsack resource at the end of the horizon. It emphasizes that the arms and knapsacks are of equal importance in determining the regret of a BwK algorithm. By further exploiting the structure of the primal and dual LPs, we develop a new sub-optimality measure for the BwK problem which can be viewed as a generalization of the sub-optimality measure for the MAB problem first derived in \citep{lai1985asymptotically}. The sub-optimality measure accounts for both arms and knapsacks, and it aims to distinguish optimal arms from non-optimal arms, and binding knapsacks from non-binding knapsacks. We use this measure as a key characterization of the hardness of a BwK problem instance.

Inspired by these findings, we propose a two-phase algorithm for the problem. The first phase of our algorithm is elimination-based and its objective is to identify the optimal arms and binding knapsacks of the BwK problem. The second phase of our algorithm utilizes the output of the first phase and it uses the optimal arms to exhaust the remaining resources through an adaptive procedure. Our algorithm and analysis feature for its full generality and we only make a mild assumption on the non-degeneracy of the underlying LP. In addition, the algorithm requires no prior knowledge of the underlying problem instance. 

\textbf{Other related literature:} \citep{agrawal2015linear, agrawal2016efficient} study the contextual BwK problem where the reward and resource consumption are both linear in a context vector. \citep{immorlica2019adversarial,kesselheim2020online} study the BwK problem under an adversarial setting. \citep{ferreira2018online} analyzes Thompson sampling-based algorithms for the BwK problem.

\section{Model and Setup}

\label{model_setup}

The problem of \textit{bandits with knapsacks} (BwK) was first defined in \cite{badanidiyuru2013bandits}, and the notations presented here are largely consistent with \cite{badanidiyuru2013bandits, agrawal2014bandits}. 
Consider a fixed and known finite set of $m$ arms (possible actions) available to the decision maker, henceforth called the algorithm. There are $d$ type of resources and a finite time horizon $T$, where $T$ is known to the algorithm. In each time step $t$, the algorithm plays an arm of the m arms, receives reward $r_t$, and consumes amount $C_{j,t} \in [0,1]$ of each resource $j \in [d]$. The reward $r_t$ and consumption
$\bm{C}_t = (C_{1,t},....,C_{d,t})^\top \in \mathbb{R}^d$ are revealed to the algorithm after choosing arm $i_t\in[m]$. The rewards and costs
in every round are generated i.i.d. from some unknown fixed underlying distribution. More precisely, there is some fixed but unknown $\bm{\mu} = (\mu_1,...,\mu_m)^\top\in \mathbb{R}^m$ and $\bm{C}= (\bm{c}_1,...,\bm{c}_m)\in \mathbb{R}^{d\times m}$ such that
$$\E[r_t|i_t] = \mu_{i_t}, \ \E[\bm{C}_{t}|i_t] = \bm{c}_{i_t}$$
where $\mu_i\in \mathbb{R}$ and $\bm{c}_i=(c_{1i},...,c_{di})^\top\in \mathbb{R}^d$ are the expected reward and the expected resource consumption of arm $i\in[m].$
In the beginning of every time step t, the algorithm needs to pick an arm $i_t$, using only the history of plays and outcomes until time step $t-1.$ There is a hard constraint capacity $B_j$ for the $j$-th type of resource. The algorithm stops at the earliest time $\tau$ when one or more of the constraints is violated, i.e. $\sum_{t=1}^{\tau+1} c_{j,t} > B_j$ for some $j\in[d]$ or if the time horizon ends, if i.e. $\tau\ge T.$ Its total reward is given by the sum of rewards
in all rounds preceding $\tau$, i.e $\sum_{t=1}^{\tau} r_t.$ The goal of the algorithm is to maximize the
expected total reward. The values of $B_j$ are known to the algorithm, and without loss
of generality, we make the following assumption.

\begin{assumption} We assume $B_j = B = \min_{j} B_j$ for all $j\in[d]$ (by scaling the consumption matrix $\bm{C}$). Let $\bm{B} = (B,...,B)^\top \in \mathbb{R}^d$. Moreover, we assume the resource capacity $\bm{B}$ scales linearly with $T$, i.e., $\bm{B}=T\cdot\bm{b}=T\cdot (b,...,b)^\top \in \mathbb{R}^d$ for some $b>0.$
\label{assume_b}
\end{assumption}
The assumption is mainly for notation simplicity and it will not change the nature of the analysis in this paper. Given that our focus is to derive asymptotic problem-dependent bound, it is natural to have the resource capacity scales linearly with the length of horizon. Throughout this paper, we use bold symbols to denote vectors/matrices and normal symbols to denote scalars.

Furthermore, without loss of generality, a ``null'' arm is introduced to represent the time constraint (time horizon $T$). Specifically, let $\mu_m=0,$ $\bm{c}_m=(b,0,...,0)^\top$, and $c_{1,i}=b$ for all $i\in[m].$ In this way, the first constraint captures the constraint of finite time horizon $T$ and the ``null'' arm can be played with no reward achieved and with no cost induced to the other factual constraints except for the time constraint.

Regret is defined as the difference in the total reward obtained by the algorithm and OPT, where OPT denotes the total expected reward for
the optimal dynamic policy. In this paper, we are interested in the (expected) problem-dependent regret,
$$\text{Regret}_T^{\pi}(\mathcal{P}, \bm{B}) \coloneqq \text{OPT} - \E\left[\sum_{t=1}^{\tau} r_t\right]$$
where $\pi$ denotes the algorithm, and $\mathcal{P}$ encapsulates all the parameters related to the distributions of reward and resource consumption, including $\bm{\mu}$ and $\bm{C}.$ The expectation is taken with respect to the randomness of the reward and resource consumption.

Consider the following linear program:
\begin{align}
  \label{primalLP} \text{OPT}_{\text{LP}} \coloneqq  \max_{\bm{x}} \ \ & \bm{\mu}^\top \bm{x} \\
    \text{s.t.}\ \ &  \bm{C} \bm{x} \le \bm{B}  \nonumber  \\
    & \bm{x}\ge \bm{0} \nonumber 
\end{align}
where the decision variables are $\bm{x} = (x_1,...,x_m)^\top \in \mathbb{R}^m.$ One can show that (see \cite{badanidiyuru2013bandits}) $$\text{OPT}_{\text{LP}} \ge \text{OPT}$$
so that $\text{OPT}_{\text{LP}}$ provides a deterministic upper bound for the expected reward under the optimal dynamic policy. Let $\bm{x}^*=(x_1^*, ...,x_m^*)^\top$ denote the optimal solution to (\ref{primalLP}). 

\section{Primal-dual Perspective for BwK} 

In this section, we present a generic regret upper bound and explore the properties of the underlying primal and dual LPs.

Let $\mathcal{I}^*$ and $\mathcal{I}'$ denote the set of optimal basic variables and the set of optimal non-basic variables of \eqref{primalLP}, respectively. Let $\mathcal{J}^*$ and $\mathcal{J}'$ denote the set of binding and non-binding constraints of \eqref{primalLP}, respectively. That is,
\begin{align*}
   \mathcal{I}^* & \coloneqq \left\{x_i^*>0, i\in [m]\right\}, \\
    \mathcal{I}' & \coloneqq \left\{x_i^*=0, i\in [m]\right\}, \\
 \mathcal{J}^* &  \coloneqq \left\{ B-\sum_{i=1}^m c_{ji}x_i^* =0, j\in[d]\right\},\\
 \mathcal{J}' &  \coloneqq \left\{ B-\sum_{i=1}^m c_{ji}x_i^*>0, j\in[d]\right\}.
\end{align*}
So, we know $\mathcal{I}^*\cap \mathcal{I}'=[m]$ and $\mathcal{J}^*\cap \mathcal{J}'=[d].$ Accordingly, we call an arm $i\in\mathcal{I}^*$ as an optimal arm and $i\in\mathcal{I}'$  as a sub-optimal arm. Here and hereafter, we will refer to knapsack as constraint so that the terminology is more aligned with the LPs.

We make the following assumption on the LP's optimal solution. 
\begin{assumption}
\label{ass_LP}
The LP \eqref{primalLP} has an unique optimal solution. Moreover, the optimal solution is non-degenerate, i.e., 
$$|\mathcal{I}^*| = |\mathcal{J}^*|.$$
\end{assumption}
The assumption is a standard one in LP's literature, and any LP can satisfy the assumption with an arbitrarily small perturbation \citep{megiddo1989varepsilon}. To interpret the non-degeneracy, consider if $|\mathcal{I}^*| = |\mathcal{J}^*|=l$, then the optimal solution to LP \eqref{primalLP} is to play the only $l$ arms in $\mathcal{I}^*$. When there is no linear dependency between the columns of $\bm{C}$, that will result in a depletion of $l$ resource constraints.

The dual problem of \eqref{primalLP} is 
\begin{align} 
\label{DualLP}
\min_{\bm{y}} \ \ & \bm{B}^\top \bm{y} \\
\text{s.t. \ } \ & \bm{C}^\top\bm{y} \ge \bm{\mu} \nonumber \\
&  \bm{y} \ge \bm{0}.\nonumber
\end{align}
Denote its optimal solution as $\bm{y}^*=(y_1^*,...,y_d^*).$ From LP's complementarity condition, we know the following relation holds under Assumption \ref{ass_LP},
\begin{align*}
   j \in \mathcal{J}^* \Leftrightarrow y_j^*>0,\ \ j \in \mathcal{J}'  \Leftrightarrow y_j^*=0.
\end{align*}

The following lemma summarizes the LP's properties.
\begin{lemma}
\label{I_J_set}
Under Assumption \ref{ass_LP}, we have the primal LP \eqref{primalLP} and the dual LP \eqref{DualLP} share the same optimal objective value. Also,
$$|\mathcal{I}^*|+|\mathcal{J}'|=d,$$
$$|\mathcal{I}'|+|\mathcal{J}^*|=m.$$
\end{lemma}

\subsection{A Generic Regret Upper Bound}

We begin our discussion with deriving a new upper bound for a generic BwK algorithm. First, we define the knapsack process as the remaining resource capacity at each time $t$. Specifically, we define $\bm{B}^{(0)} \coloneqq \bm{B}$ and 
$$\bm{B}^{(t+1)} \coloneqq \bm{B}^{(t)} - \bm{C}_t$$
for $t\in[T].$ Recall that $\bm{C}_t$ is the (random) resource consumption at time $t$. The process $\bm{B}^{(t)}$ is pertaining to the BwK algorithm. In addition, we follow the convention of the bandits literature and define the count process $n_{i}(t)$ as the number of times the $i$-th arm is played up to the end of time period $t$. 

\begin{proposition}
\label{prop_upper_bound}
The following inequality holds for any BwK algorithm,
\begin{align}
    \text{Regret}_T^{\pi}(\mathcal{P}, \bm{B}) \le   \sum_{i\in \mathcal{I}'} n_i(t)\Delta_i +\mathbb{E}\left[\bm{B}^{(\tau)}\right]^\top\bm{y}^*. \label{upperBound}
\end{align}
where $\Delta_i = \bm{c}_i^\top\bm{y}^*-\mu_i$ for $i\in[m].$
\end{proposition}
Here $\Delta_i$ is known as reduced cost/profit in LP literature and it quantifies the cost-efficiency of each arm (each basic variable in LP). The upper bound in Proposition \ref{prop_upper_bound} is new to the existing literature and it can be generally applicable to all BwK algorithms. It  consists of two parts: (i) the number of times for which a sub-optimal arm is played multiplied by the corresponding reduced cost; (ii) the remaining resource at time $\tau$, either when any of the resource is depleted or at the end of time horizon $T$.  The first part is consistent with the classic bandits literature in that we always want to upper bound the number of sub-optimal arms being played throughout the horizon. At each time a sub-optimal arm $i\in\mathcal{I}'$ is played, a cost of $\Delta_i$ will be induced. Meanwhile, the second part is particular to the bandits with knapsacks setting and can easily be overlooked. Recall that the definition of $\tau$ refers to the first time that any resource $j\in[d]$ is exhausted (or the end of the horizon). It tells that the left-over of resources when the process terminates at time $\tau$ may also induce regret. For example, for two binding resources $j,j'\in \mathcal{J}^*$, it would be less desirable for one of them $j$ to have a lot of remaining while the other one $j'$ is exhausted. Since the binding resources are critical in determining optimal objective value for LP \eqref{primalLP}, intuitively, it is not profitable to waste any of them at the end of the procedure.

\subsection{Symmetry between Arms and Bandits}

\label{symmetry}

From Proposition \ref{prop_upper_bound}, we see the importance of dual problem \eqref{DualLP} in bounding an algorithm's regret. Now, we pursue further along the path and propose a new notion of sub-optimality for the BwK problem. Our sub-optimality measure is built upon both the primal and dual LPs, and it reveals the combinatorial structure of the problem. In the following, we define two classes of LPs, one for the arm $i\in[m]$ and the other for the constraints $j\in[d].$

First, for each arm $i\in[m],$ define
\begin{align} 
\text{OPT}_i \coloneqq \max_{\bm{x}} \ \ & \bm{\mu}^\top \bm{x}, \label{opt_i} \\
\text{s.t. \ } \ & \bm{C}\bm{x} \le \bm{B},\nonumber \\
&  x_i=0, \bm{x} \ge \bm{0}. \nonumber
\end{align}
By definition, $\text{OPT}_i$ denotes the optimal objective value of an LP that takes the same form as the primal LP \eqref{primalLP} except with an extra constraint $x_i=0$. It represents the optimal objective value if the $i$-th arm is not allowed to use. For a sub-optimal arm $i\in\mathcal{I}',$ $\text{OPT}_i=\text{OPT}_{\text{LP}}$, while for an optimal arm $i\in\mathcal{I}^*$, $\text{OPT}_i<\text{OPT}_{\text{LP}}$. In this way, $\text{OPT}_i$ characterizes the importance of arm $i$.

Next, for each constraint $j\in[d]$, define
\begin{align}
\text{OPT}_j \coloneqq \min_{\bm{y}} \ \ \label{opt_j}& \bm{B}^{\top}\bm{y} - B, \\
\text{s.t.}\ \ & {{\bm{C}}}^{\top}\bm{y} \geq \bm{\mu} +\bm{C}_{j,\cdot}, \nonumber \\
& \bm{y}\geq\bm{0},\nonumber
\end{align}   
where $\bm{C}_{j,\cdot}$ denotes the $j$-th row of the constraint matrix $\bm{C}.$ Though it may not be as obvious as the previous case of OPT$_i$, the definition of $\text{OPT}_j$ aims to characterize the bindingness/non-bindingness of a constraint $j$. The point can be illustrated by looking at the primal problem for \eqref{opt_j}. From LP's strong duality, we know 
    \begin{align}
        \label{p-opt_j}
     \text{OPT}_j =  \max_{\bm{x}} \ \  & \bm{\mu}^\top \bm{x} -(B-\bm{C}_{j,\cdot}^{\top}\bm{x}), \\
        \text{s.t.}\ \ & \sum_{i=1}^{m} \bm{C} \bm{x} \leq \bm{B}, \nonumber \\
        & \bm{x} \geq 0. \nonumber
\end{align}    
Compared to the original primal LP \eqref{primalLP}, there is an extra term in the objective function in LP \eqref{p-opt_j}. The extra term is a penalization for the left-over of the $j$-th constraint, and thus it encourages the usage of the $j$-th constraint. For a binding constraint $j\in\mathcal{J}^*,$ it will be exhausted under the optimal solution $\bm{x}^*$ to LP \eqref{primalLP} so the penalization term does not have any effect, i.e., \text{OPT}$_j=$\text{OPT}$_{\text{LP}}$. In contrast, for a non-binding constraint $j\in\mathcal{J}',$ the extra term will result in a reduction in the objective value, i.e., $\text{OPT}_j<\text{OPT}_{\text{LP}}$. We note that one can introduce any positive weight to the penalization term so as to trade off between the reward and the left-over of the $j$-th constraint in \eqref{p-opt_j}, but its current version suffices our discussion. 

The following proposition summarizes the properties of OPT$_i$ and OPT$_j.$
\begin{proposition}
\label{OPTij}
Under Assumption \ref{ass_LP}, we have 
\begin{align*}
\text{OPT}_i<\text{OPT}_{\text{LP}} &\Leftrightarrow i\in\mathcal{I}^*, \\
\text{OPT}_i=\text{OPT}_{\text{LP}} &\Leftrightarrow i\in\mathcal{I}', \\
\text{OPT}_j=\text{OPT}_{\text{LP}} &\Leftrightarrow j\in\mathcal{J}^*, \\
\text{OPT}_j<\text{OPT}_{\text{LP}} &\Leftrightarrow j\in\mathcal{J}'.
\end{align*}

\end{proposition}

In this way, the definition of $\text{OPT}_i$ distinguishes optimal arms $\mathcal{I}^*$ from sub-optimal arms $\mathcal{I}'$, while the definition of $\text{OPT}_j$ distinguishes the binding constraints $\mathcal{J}^*$ from non-binding constraints $\mathcal{J}'$. The importance of such a distinguishment arises from the upper bound in Proposition \ref{prop_upper_bound}: on one hand, we should avoid playing sub-optimal arms, and on the other hand, we should exhaust the binding resources. A second motivation for defining both OPT$_i$ and OPT$_j$ can be seen after we present our algorithm. Furthermore, we remark that a measurement of the sub-optimality of the arms has to be defined through the lens of LP due to the combinatorial nature of the problem. The effect of the $i$-th arm's removal on the objective value can only be gauged by solving an alternative LP of OPT$_i.$ Similarly, a measurement of the bindingness of the constraints should also take into account the combinatorial relation between constraints. 

Next, define
$$\delta \coloneqq \frac{1}{T} \left(\text{OPT}_{\text{LP}} - \max\left\{\max_{i \in \mathcal{I}^*} \text{OPT}_i, \max_{j\in \mathcal{J}'} \text{OPT}_j\right\}\right).$$
where the factor $\frac{1}{T}$ is to normalize the optimality gap by the number of time periods. Under Assumption \ref{assume_b}, all the objective values in above should scale linearly with $T.$  

To summarize, $\delta$ characterizes the hardness of distinguishing optimal arms from non-optimal arms (and binding constraints from non-binding constraints). It can be viewed as a generalization of the sub-optimality measure $\delta_{\text{MAB}} = \min_{i\neq i^*} \mu_{i^*} - \mu_i$ for the MAB problem \citep{lai1985asymptotically}. $\delta_{\text{MAB}}$ characterizes the hardness of an MAB problem instance, i.e., the hardness of distinguishing the optimal arm from sub-optimal arms. In the context of BwK, $\delta$ is a more comprehensive characterization in that it takes into account both the arms and the constraints. Imaginably, it will be critical in both algorithm design and analysis for the BwK problem.

\subsection{Key Parameters of the LPs}
\label{param_def}

Now, we define two LP-related quantities that will appear in our regret bound:

\begin{itemize}
    \item Linear Dependency between Arms: Define $\sigma$ be the minimum singular value of the matrix $\bm{C}_{\mathcal{I}^*, \mathcal{J}^*}.$ Specifically, 
    $$\sigma \coloneqq \sigma_{\min} \left(\bm{C}_{\mathcal{J}^*, \mathcal{I}^*}\right) = \sigma_{\min} \left((c_{ji})_{j\in\mathcal{J}^*, i\in\mathcal{I}^*}\right).$$
    In this light, $\sigma$ represents the linear dependency between optimal arms across the binding constraints. For a smaller value of $\sigma$, the optimal arms are more linearly dependent, and then it will be harder to identify the optimal numbers of plays. Under Assumption \ref{ass_LP}, the uniqueness of optimal solution implies $\sigma>0.$
    \item Threshold on the optimal solution:
    \begin{align*}
    \chi &\coloneqq \frac{1}{T}\cdot \min \{x_i^*\neq 0,i\in[m]\}  
    \end{align*}
    If the total resource $\bm{B}=T\cdot \bm{b},$ both the optimal solution $(x_1^*,...,x_m^*)$ should scale linearly with $T$. The factor $\frac{1}{T}$ normalizes the optimal solution into a probability vector. $\chi$ denotes the smallest non-zero entry for the optimal solution. Intuitively, a small value of $\chi$ implies that the optimal proportion of playing an arm $i\in\mathcal{I}^*$ is small and thus it is more prone to ``overplay'' the arm. 
\end{itemize}

\textbf{Remarks.} By the definition, it seems that the above parameters $\delta$ and $\chi$ both involve a factor of $T$. But if we replace $\bm{B}$ (the right-hand-side of LP) with $\bm{b}$ from Assumption \ref{assume_b}, then the factor $T$ disappears, and the parameters $\chi$ and $\delta$ are essentially dependent on $\bm{\mu}$, $\bm{C}$, and $\bm{b}$ which are inherent to the problem instance but bear no dependency on the horizon $T$. In other words, Assumption \ref{assume_b} frees the dependency on $T$ by introducing the quantity $b$. Practically, the assumption states the resource capacity should be sufficiently large and it is natural in many application contexts (for example, the small bids assumption in AdWords problem \cite{mehta2005adwords}). Theoretically, in two previous works \cite{flajolet2015logarithmic, sankararaman2020advances}, either a factor of $1/T$ appears in the parameter definition \cite{sankararaman2020advances} or the assumption is explicitly imposed \cite{flajolet2015logarithmic}. Such an assumption might be inevitable for a logarithmic regret to be derived. 

\subsection{LCB and UCB}

Throughout this paper, we denote the reward and resource consumption of the $s$-th play of the $i$-th arm as 
$r_{i,s}$ and $\bm{C}_{i,s}=(C_{1i,s},...,C_{di,s})^\top$
respectively, for $i \in[m]$ and $s\in[T].$ Let $n_{i}(t)$ be the number of times the $i$-th arm is played in the first $t$ time periods. Accordingly, we denote the estimators at time $t$ for the $i$-th arm  as 
$$\hat{\mu}_{i}(t) \coloneqq \frac{1}{n_{i}(t)}\sum_{s=1}^{n_{i}(t)} r_{i,s},$$
$$\hat{C}_{ji}(t) \coloneqq \frac{1}{n_{i}(t)} \sum_{s=1}^{n_{i}(t)} C_{ji,s}$$
for $i\in[m]$ and $j\in[d].$ In a similar manner, we define the estimator for the $i$-th arm's resource consumption vector as $\hat{\bm{C}}_{i}(t) \coloneqq (\hat{C}_{1i}(t),...,\hat{C}_{di}(t))^\top$, and the resource consumption matrix as $\hat{\bm{C}}(t) \coloneqq (\hat{\bm{C}}_{1}(t),...,\hat{\bm{C}}_{m}(t))$. Specifically, without changing the nature of the analysis, we ignore the case that when $n_{i}(t)=0$. Then, we define the lower confidence bound (LCB) and upper confidence bound (UCB) for the parameters as 
\begin{align*}
    \mu_i^L(t)& \coloneqq  proj_{[0,1]}\left(\hat{\mu}_{i}(t)-\sqrt{\frac{2\log T}{n_i(t)}}\right) \\
    \mu_i^U(t)&\coloneqq  proj_{[0,1]}\left(\hat{\mu}_{i}(t)+\sqrt{\frac{2\log T}{n_i(t)}}\right)\\
    C_{ji}^L(t)& \coloneqq proj_{[0,1]}\left(\hat{C}_{ji}(t)-\sqrt{\frac{2\log T}{n_i(t)}}\right) \\
    C_{ji}^U(t)& \coloneqq proj_{[0,1]}\left(\hat{C}_{ji}(t)+\sqrt{\frac{2\log T}{n_i(t)}}\right)\\
\end{align*}
where $proj_{[0,1]}(\cdot)$ projects a real number to interval $[0,1].$ The following lemma is standard in bandits literature and it characterizes the relation between the true values and the LCB/UCB estimators. It states that all the true values will fall into the intervals defined by the corresponding estimators with high probability.

\begin{lemma}[Concentration]
\label{paramEst}
The following event holds with probability no less than $1-\frac{4md}{T^2}$,
\begin{align*}
\mu_i&\in\left(\mu_i^L(t), \mu_i^U(t)\right),\\
c_{ji}&\in\left(C_{ji}^L(t), C_{ji}^U(t)\right), 
\end{align*}
for all $i\in[m],j\in[d],t\in T$.
\end{lemma}

With the UCB/LCB estimators for the parameters, we can construct UCB/LCB estimators for the objective of the primal LP. Specifically, 
\begin{align}
\label{UCB_LP}
 \text{OPT}_{\text{LP}}^U \coloneqq  \max_{\bm{x}} \ \ & \left(\bm{\mu}^U\right)^\top \bm{x}, \\
    \text{s.t.}\ \ &  \bm{C}^L \bm{x} \le \bm{B},  \nonumber  \\
    & \bm{x}\ge \bm{0}. \nonumber\\
\label{LCB_LP} \text{OPT}_{\text{LP}}^L \coloneqq  \max_{\bm{x}} \ \ & \left(\bm{\mu}^L\right)^\top \bm{x}, \\
    \text{s.t.}\ \ &  \bm{C}^U \bm{x} \le \bm{B},  \nonumber  \\
    & \bm{x}\ge \bm{0}. \nonumber 
\end{align}
The following lemma states the relation between $\text{OPT}_{\text{LP}}^U$, $\text{OPT}_{\text{LP}}^L$, and $\text{OPT}_{\text{LP}}.$ Intuitively, if we substitute the original constraint matrix $\bm{C}$ with its LCB (or UCB) and the objective coefficient $\bm{\mu}$ with its UCB (or LCB), the resultant optimal objective value will be a UCB (or LCB) for $\text{OPT}_{\text{LP}}$. 

\begin{lemma}
\label{optl_u}
The following inequality holds with probability no less $1-\frac{4md}{T^2},$
$$\text{OPT}_{\text{LP}}^L \le \text{OPT}_{\text{LP}} \le \text{OPT}_{\text{LP}}^U.$$
\end{lemma}
A similar approach is used in \citep{agrawal2014bandits} to develop an UCB-based algorithm for the BwK problem. For our algorithm presented in the following section, we will construct estimates not only for the primal LP \eqref{primalLP}, but also for the LPs \eqref{opt_i} and \eqref{opt_j}. By comparing the estimates of OPT$_{\text{LP}}$, OPT$_{i}$, and OPT$_{j}$, we will be able to identify the optimal arms $\mathcal{I}^*$ and the non-binding constraints $\mathcal{J}'.$

\section{Two-Phase Algorithm}

In this section, we describe our two-phase algorithm for the BwK problem. The main theme is to use the underlying LP's solution to guide the plays of the arms. The two phases in the algorithm correspond to the two parts of the regret upper bound in Proposition \ref{prop_upper_bound}. In the following, we describe the two phases of the algorithm and their intuitions respectively.

Phase I of Algorithm \ref{alg_BwK} is an elimination algorithm and it aims to identify the optimal arms $\mathcal{I}^*$ and the non-binding constraints $\mathcal{J}'.$ In each round of the while loop in Phase I, all the arms are played once to improve the estimators for $\mu$ and $\bm{C}.$ After each round of plays, an identification procedure is conducted by comparing the LCB of the original optimal value ($\text{OPT}^L_{\text{LP}}$) against the UCBs ($\text{OPT}^U_{i}$ and $\text{OPT}^U_{j}$). Recall that Proposition \ref{OPTij}, there will be a non-zero sub-optimality gap if $i\in\mathcal{I}^*$ or $j\in\mathcal{J}'.$ So, if the algorithm observes a gap between the corresponding LCBs/UCBs, it will assert $i\in\mathcal{I}^*$ or $j\in\mathcal{J}'$, and the assertion will be true with high probability. 

The stopping rule in Phase I originates from the complementary condition in Lemma \ref{I_J_set}, i.e., $|\hat{\mathcal{I}}^*|+|\hat{\mathcal{J}}'|<d$. This is a key in the primal-dual design of the algorithm and it further justifies the consideration of the dual problem. Without maintaining the set $\hat{\mathcal{J}}',$ we cannot decide whether we have obtain the true primal set, i.e., $\hat{\mathcal{I}}^*=\mathcal{I}^*$. Specifically, since there is no precise knowledge of the number of optimal arms $|\mathcal{I}^*|,$ while we keep adding arms into $\hat{\mathcal{I}}^*$, we do not know when to stop. The complementarity in Lemma  \ref{I_J_set} provides a condition on the number of arms in $\mathcal{I}^*$ and the number of constraints in $\mathcal{J}'$. Accordingly, Phase I is terminated when this condition is met. Moreover, we emphasize that a by-product of the Phase I is a \textit{best arm identification} procedure. To the best of our knowledge, this is the first result on identifying optimal arms for the BwK problem.

\begin{algorithm}[ht!]
\caption{Primal-dual Adaptive Algorithm for BwK}
\label{alg_BwK}
\begin{algorithmic}[1] 
\State Input: Resource capacity $\bm{B}$, $T$
\State \textcolor{blue}{\%\% Phase I: Identification of $\mathcal{I}^*$ and $\mathcal{J}'$}
\State Initialize $\hat{\mathcal{I}}^*=\hat{\mathcal{J}}'=\emptyset$, $t=0$
\State Initialize the knapsack process $\bm{B}^{(0)} = \bm{B}$
\While{$|\hat{\mathcal{I}}^*|+|\hat{\mathcal{J}}'|<d$}
\State Play each arm $i\in[m]$ once 
\State Update $t=t+m$ and the knapsack process $\bm{B}^{(t)}$
\State Update the estimates $\hat{\bm{\mu}}(t)$ and $\hat{\bm{C}}(t)$
\State Solve the LCB problem \eqref{LCB_LP} and obtain OPT$_{\text{LP}}^L(t)$
\For{$i\notin \hat{\mathcal{I}}^*$}
\State Solve the following UCB problem for OPT$_{i}$
\begin{align*}
 \text{OPT}_{i}^U(t) \coloneqq  \max_{\bm{x}} \ \ & \left(\bm{\mu}^U(t)\right)^\top \bm{x}, \\
    \text{s.t.}\ \ &  \bm{C}^L(t) \bm{x} \le \bm{B},  \nonumber  \\
    & x_i=0, \bm{x}\ge \bm{0}. \nonumber
    \end{align*}   
\If{ $\text{OPT}_{\text{LP}}^L(t) > \text{OPT}_i^{\text{U}}(t)$ }
\State Update $\hat{\mathcal{I}}^* = \hat{\mathcal{I}}^* \cup \{i\}$
\EndIf
\EndFor
\For{$j \notin \hat{\mathcal{J}}'$}
\State Solve the following UCB problem for OPT$_{j}$
    \begin{align*}
        \text{OPT}^{U}_j(t) :=\min_{\bm{y}} \ \ & \bm{B}^{\top}\bm{y} -B\nonumber,\\
        \text{s.t.}\ \ & ({{\bm{C}}}^{L}(t))^{\top}\bm{y} \geq {\bm{\mu}}^{U}(t) +\bm{C}_{j,\cdot}^{U}(t),\\
        &\bm{y}\ge 0. \nonumber
    \end{align*} 
\If{ $\text{OPT}_{\text{LP}}^L(t) > \text{OPT}_j^{\text{U}}(t)$ }
\State Update $\hat{\mathcal{J}}' = \hat{\mathcal{J}}' \cup \{j\}$
\EndIf
\EndFor
\EndWhile
\State Update $t=t+1$
\State \textcolor{blue}{\%\% Phase II: Exhausting the binding resources}
\While{$t\le \tau$}
\State Solve the following LP
\begin{align}
    \max_{\bm{x}} \ \ & \left(\bm{\mu}^U(t-1) \right)^\top \bm{x}, \label{adaptLP} \\
    \text{s.t.}\ \ &  \bm{C}^L(t-1) \bm{x} \le \bm{B}^{(t-1)},  \nonumber  \\
    & x_i=0, \ i\notin \mathcal{I}^*, \nonumber \\
    & \bm{x}\ge \bm{0}. \nonumber
\end{align}
\State Denote its optimal solution as $\tilde{\bm{x}}$
\State Normalize $\tilde{\bm{x}}$ into a probability and randomly play an arm according to the probability
\State Update estimates $\hat{\bm{\mu}}(t)$, $\hat{\bm{C}}(t)$, and $\bm{B}^{(t)}$
\State Update $t=t+1$
\EndWhile
\end{algorithmic}
\end{algorithm}

Phase II of Algorithm \ref{alg_BwK} is built upon the output of Phase I. At each time $t$, the algorithm solves an adaptive version LP \eqref{adaptLP} and normalizes its optimal solution into a sampling scheme on arms. Also, in Phase II, the algorithm will only play arms $i\in\hat{\mathcal{I}}^*$; this is achieved by enforcing $x_i=0$ for $i\in\hat{\mathcal{I}}^*$ in \eqref{adaptLP}. The adaptive design is exemplified on the right-hand-side of the LP \eqref{adaptLP}, where instead of the static resource capacity $\bm{B}$, it uses the remaining resource at the end of last time period $\bm{B}^{(t-1)}$. To see its intuition, consider if a binding resource $j\in \mathcal{J}^*$ is over-used in the first $t$ time periods, then it will result in a smaller value of $B^{(t-1)}_j$, and then the adaptive mechanism will tend to be more reluctant to consume the $j$-th resource in the future, and vice versa for the case of under-use.

We emphasize that the adaptive design is not only intuitive but also necessary to achieve a regret that is logarithmic in $T$. If we adopt a static right-hand-side $\bm{B}$ in \eqref{adaptLP} (as in \citep{badanidiyuru2013bandits, agrawal2014bandits}), then the fluctuation of the process $\bm{B}_t$ will be on the order of $\Omega(\sqrt{t})$. Consequently, when it approaches to the end of the horizon, certain type of binding resource may be exhausted while other binding resources still have $\Omega(\sqrt{T})$ left-over, and this may result in an $\Omega(\sqrt{T})$ upper bound in Proposition \ref{prop_upper_bound}. The intuition is made rigorous by \citep{arlotto2019uniformly}; the paper establishes that without an adaptive design, the regret is at least $\Omega(\sqrt{T})$ for the multi-secretary problem (can be viewed as a one-constraint BwK problem) even if the underlying distribution is known.

\section{Regret Analysis}

In this section, we derive an regret upper bound for Algorithm \ref{alg_BwK} by analyzing the two phases separately. 

\subsection{Analysis of Phase I}

Proposition \ref{PhaseI_UB} provides an upper bound on the number of time periods within which Phase I will terminate. It also states that the identification of $\mathcal{I}^*$ and $\mathcal{J}'$ will be precise conditional on the high probability event in Lemma \ref{paramEst}.

\begin{proposition}
\label{PhaseI_UB}
In Phase I of Algorithm \ref{alg_BwK}, each arm $i\in[m]$ will be played for no more than $\left(2+\frac{1}{b}\right)^2\cdot\frac{72\log T}{\delta^2}$ times where $b$ is defined in Assumption \ref{assume_b}. If the resources are not exhausted in Phase I, then its output satisfies
$$\prob\left(\hat{\mathcal{I}}^*=\mathcal{I}^*, \hat{\mathcal{J}}'=\mathcal{J}'\right)\ge 1-\frac{4md}{T^2}.$$
\end{proposition}

The surprising point of Proposition \ref{PhaseI_UB} lies in that there are $O(2^{m+d})$ possible configurations of $(\mathcal{I}^*, \mathcal{J}')$ and, without any prior knowledge, the true configuration can be identified within $O(\log T)$ number of plays for each arm. In contrast, \citep{flajolet2015logarithmic} does not utilize the primal-dual structure of the problem and conducts a brute-force search in all possible configurations which results in an $O(2^{m+d} \log T)$ regret. In addition, the search therein requires the knowledge of a non-degeneracy parameter a priori. The result also explains why \citep{sankararaman2020advances} imposes a single-best-arm condition for the BwK problem, which assumes $\mathcal{I}^*=1$. This additional greatly simplifies the combinatorial structure and reduces the BwK problem more closely to the standard MAB problem. In this light, Proposition \ref{PhaseI_UB} can be viewed as a best-arm-identification result \citep{audibert2010best} for the BwK problem in full generality and without any prior knowledge.

\subsection{Analysis of Phase II}

Proposition \ref{prop_B_tau} provides an upper bound on the remaining resource for binding constraints when the procedure terminate, which corresponds to the second part of Proposition \ref{prop_upper_bound}. Notably, the upper bound has no dependency on $T$. In a comparison with the $\Omega(\sqrt{T})$ fluctuation under the static design, it demonstrates the effectiveness of our adaptive design.

\begin{proposition}
\label{prop_B_tau}
For each binding constraint $j\in \mathcal{J}^*,$ we have
$$\E\left[B_{j}^{(\tau)}\right] = O\left(\frac{d^3}{b\min\{\chi^2, \delta^2\}\min\{1,\sigma^2\}}\right)$$
where $\chi$ and $\sigma$ are defined in Section \ref{param_def}, and $b$ is defined in Assumption \ref{assume_b}.
\end{proposition}

The idea of proof is to introduce an auxiliary process $b_j^{(t)}=\frac{B_j^{(t)}}{T-t}$ for $t\in[T-1]$ and $j\in\mathcal{J}^*$. Recall that $B_j^{(t)}$ is the $j$-th component of the knapsack process $\bm{B}^{(t)}$, we know its initial value $b_j^{(0)}=b$. Then define 
$$\tau_j = \min\{t: b_j^{(t)}\notin [b-\epsilon, b+\epsilon]\} \cup \{T\}$$
for a fixed $\epsilon>0.$ With the definition, $b_j^{(t)}$ can be interpreted as average remaining resource (per time period) and $\tau_j$ can be interpreted as the first time that $b_j^{(t)}$ deviates from its initial value $b$ by a small amount. It is easy to see that $\tau_j\le \tau$ with a proper choice of $\epsilon.$ Next, we aim to upper bound $\E[T-\tau_j]$ by analyzing the process $\{b_j^{(t)}\}_{t=0}^T.$ From the dynamic of the knapsack process $\bm{B}_t,$ we know that
\begin{align}
   b_j^{(t)} & =  \frac{B_j^{(t)}}{T-t} = \frac{B_j^{(t-1)}-C_{j,t}}{T-t} \nonumber \\
   & =  b_j^{(t-1)} - \frac{1}{T-t}\left(C_{j,t}-b_j^{(t-1)}\right) \label{dynamic_b}
\end{align}
where $C_{j,t}$ as defined earlier is the resource consumption of $j$-th constraint at time $t$. The above formula \eqref{dynamic_b} provides a technical explanation for the motivation of the adaptive design in \eqref{adaptLP}. Intuitively, when the right-hand-side of \eqref{adaptLP} is $\bm{B}^{(t-1)},$ it will lead to a solution that (approximately and on expectation) consumes $b_j^{(t-1)}$ of the $j$-th resource for each of the following time periods. Ideally, this will make the second term in \eqref{dynamic_b} have a zero expectation. However, due to estimation error for the LP's parameters, this may not be the case. The idea is to first provide an upper bound for the ``bias'' term
$$\left\vert\E\left[C_{j,t}-b_j^{(t-1)}|\mathcal{H}_{t-1}\right]\right|$$
where $\mathcal{H}_{t-1}=\{(r_s, \bm{C}_s, i_s)\}_{s=1}^{t-1}$ encapsulates all the information up to time $t-1.$ Unsurprisingly, the upper bound is on the order of $O\left(\frac{1}{\sqrt{t}}\right)$. Next, with this bias upper bound, we can construct a super-martingale (sub-martingale) based on the dynamic \eqref{dynamic_b} and employ Azuma–Hoeffding inequality to provide a concentration result for the value of the martingale. Through the analysis of the process $b_j^{(t)}$, we can derive an upper bound for $\mathbb{E}[T-\tau_j]$, and consequently, it leads to an upper bound on $\mathbb{E}[B^{(\tau)}_j]$.

The importance of the adaptive design has been widely recognized in other constrained online learning problems, such as online matching problem \cite{manshadi2012online}, online assortment problem \cite{golrezaei2014real}, online linear programming problem \cite{li2019online}, network revenue management problem \cite{jasin2012re}, etc. The common pattern of these problem is to allocate limited resources in a sequential manner, and the idea of adaptive design is to adjust the allocation rule dynamically according to the remaining resource/inventory. This is in parallel with LP \eqref{adaptLP} where the solution at each time $t$ is contingent on the remaining resources $\bm{B}^{(t)}.$ The significance of our algorithm design and analysis lies in that (i) to the literature of BwK, our paper is the first application of the idea of adaptive design; (ii) to the existing literature of adaptive design in constrained online learning problems, our work provides its first application and analysis in a partial-information environment. For the second aspect, all the existing analysis on the adaptive design fall in the paradigm of ``first-observe-then-decide'' while the BwK problem is ``first-decide-then-observe''. Specifically, in matching/resource allocation/revenue management problems, at each time period, a new agent arrives, and upon the observation, we decide the matching for the agent; or a new customer arrives, and upon the observation of her preference, we decide the assortment decision for the customer. So, the existing analyses are analogous to a BwK ``setting'' where the reward and resource consumption of playing an arm are first observed (magically), and then we decide whether we want to play the arm or not.

\subsection{Regret Upper Bound for Algorithm \ref{alg_BwK}}

Combining the two parts, we have the following result on the regret of Algorithm \ref{alg_BwK}.

\begin{proposition}
\label{regret_prop}
The regret of Algorithm \ref{alg_BwK} has the following upper bound, $$O\left(\left(2+\frac{1}{b}\right)^2\frac{md\log T}{b\delta^2} + \frac{d^4}{b^2\min\{\chi^2, \delta^2\}\min\{1,\sigma^2\}}\right)$$
where $b$ is defined in Assumption \ref{assume_b}, $\delta$ is defined in Section \ref{symmetry}, and $\sigma$ and $\chi$ are defined in Section \ref{param_def}.
\end{proposition}

The result reduces the exponential dependence on $m$ and $d$ in \citep{flajolet2015logarithmic} to polynomial, and also it does not rely on any prior knowledge. Specifically, the authors consider several settings for BwK problem, many of which assume special structures such as one or two resource constraints. The most general settings therein, which is comparable to ours, allows arbitrary number of constraints and number of optimal arms. In terms of the key parameters, the way we define $\delta$ is the same as their definition of $\Delta_x$. However, the regret bound (Theorem 8 therein) involves a summation of exponentially many $\frac{1}{\Delta_x}$'s (the same as the total number of the bases of the LP). Our parameter $\sigma$ is related to their $\epsilon$ (Assumption 8 therein) while the latter is more restrictive. Because $\sigma$ in our paper represents the minimal singular value of the matrix corresponding to only the optimal basis of the primal LP, whereas the parameter $\epsilon$ therein represents a lower bound of the determinant of the matrices corresponding to all the possible (exponentially many) bases of the primal LP. In this light, if they adopt our parameter $\sigma$, their bound would be improved by a factor of $d!$ ($C!$ therein). Moreover, $\epsilon$ is a lower bound for our parameter $\chi$ and \citep{flajolet2015logarithmic} explicitly requires the knowledge of $\epsilon$ a priori.

The proposition also generalizes the one-constraint bound in \citep{sankararaman2020advances} and relaxes the deterministic resource consumption assumption therein. Specifically, the authors assume there is one single optimal arm and one single resource (other than time), i.e., the optimal solution to the primal LP has only one non-zero entry ($|\mathcal{I}^*|=|\mathcal{J}^*|=1$ and $d=2$). They also assume the underlying LP is non-degenerate. Our results generalize their work in allowing arbitrary $d$, $|\mathcal{I}^*|$ and $|\mathcal{J}^*|.$ In terms of the key parameters, under their assumption, our parameter $\sigma=1$ because it is defined by a 1-by-1 matrix. Our parameter $\chi$ is deemed as a constant in their paper. For $\delta$, under their assumptions, our definition of OPT$_i$ can be adjusted accordingly. Specifically, we can replace the constraint $x_i=0$ in defining OPT$_i$ with $x_{i'}=0, i'\neq i$. Then our definition of $\delta$ would reduce to their definition of $G_{LAG}(a)$, both of which characterize the sub-optimality gap of an arm. 

The above comparison against the existing literature highlights that the parameters $\sigma,$ $\delta$, and $\chi$ or other LP-based parameters might be inevitable in characterizing logarithmic regret bound. Our paper makes some preliminary efforts along this line, but we do not believe our bound is tight: parameters such as $\chi$ and $\sigma$ may be replaced with some tighter parameter through a better algorithm and sharper analysis. As remarked in Section \ref{param_def}, the parameters are not dependent on $T$ under Assumption \ref{assume_b}. But to characterize the dependency of these parameters (such as $\delta$ and $\chi$) on the problem size $m$ and $d$ remains an open question. Moreover, our algorithm and analysis highly rely on the structural properties of the underlying LP, which might not be the unique method to handle the BwK problem.

\section{Conclusion}

In this paper, we introduce a new BwK algorithm and derive problem-dependent bound for the algorithm. In the Phase I of the algorithm, it involves a round-robin design and may result in playing sub-optimal arms for inefficiently many times. Our regret bound can be large when the parameters $\sigma,$ $\delta$, and $\chi$ are small and the inefficient plays of the sub-optimal arms may prevent the algorithm from achieving an $O(\sqrt{T})$ worst-case regret. In the extreme case, these parameters may scale with $O(\frac{1}{T})$, though Assumption \ref{assume_b} prevents such a possibility. So the question is whether Assumption \ref{assume_b} is necessary in admitting a logarithmic problem-dependent bound.

We conclude our discussion with a new one-phase algorithm -- Algorithm \ref{alg_Adapt}. The algorithm is also LP-based, and at each time $t$, it solves a UCB version of the primal LP to sample the arm. The algorithm has an adaptive design to exhaust the resources. On one hand, the algorithm naturally incorporates the Phase I of Algorithm \ref{alg_BwK} as a part of its Phase II. It directly enters the Phase II of Algorithm \ref{alg_BwK} and lets the adaptive LP to fully determine the arm(s) to play (without the extra constraint in \eqref{adaptLP}). On the other hand, the algorithm can be viewed as an adaptive version of the algorithm in \citep{agrawal2014bandits}. Our conjecture is that Algorithm \ref{alg_Adapt} is the optimal algorithm for BwK: it is optimal in the sense that it achieves optimal problem-dependent bound, but also admits $O(\sqrt{T})$ problem-independent bound. Unfortunately, its analysis is more challenging than Algorithm \ref{alg_BwK}, which we leave as an open question. 

\begin{algorithm}[ht!]
\caption{One-Phase Adaptive Algorithm for BwK}
\label{alg_Adapt}
\begin{algorithmic}[1] 
\State Input: Resource capacity $\bm{B}$, $T$
\State Initialize the knapsack process $\bm{B}^{(0)} = \bm{B}$
\State Initialize the estimates $\hat{\bm{\mu}}(0)$ and $\hat{\bm{C}}(0)$
\State Set $t=1$
\While{$t\le \tau$}
\State Solve the following LP
\begin{align}
    \max_{\bm{x}} \ \ & \left(\bm{\mu}^U(t-1) \right)^\top \bm{x}, \label{adapt_LP} \\
    \text{s.t.}\ \ &  \bm{C}^L(t-1) \bm{x} \le \bm{B}^{(t-1)},  \nonumber  \\
    & \bm{x}\ge \bm{0}. \nonumber
\end{align}
\State Denote its optimal solution as $\tilde{\bm{x}}$
\State Normalize $\tilde{\bm{x}}$ into a probability and randomly play an arm according to the probability
\State Update estimates $\hat{\bm{\mu}}(t)$, $\hat{\bm{C}}(t)$, and $\bm{B}^{(t)}$
\State Update $t=t+1$
\EndWhile
\end{algorithmic}
\end{algorithm}

\section*{Acknowledgements}
We would like to thank anonymous reviewers for their helpful suggestions and insightful comments.

\bibliographystyle{informs2014} % outcomment this and next line in Case 1
\bibliography{sample.bib} % if more than one, comma separated

\newpage

\appendix

\section{Proof of Section 2 and 3}

\renewcommand{\thesubsection}{A\arabic{subsection}}

\subsection{Proof of Lemma \ref{I_J_set}}
\begin{proof}
    Notice $\bm{x}=\bm{0}$ is a feasible solution to the primal LP \eqref{primalLP} and $\bm{y}=(b,...,b)^\top$ is a feasible solution to the dual LP \eqref{DualLP}. We have both of the primal and dual LPs are feasible. So, by the strong duality of linear programming, both of the primal and dual LPs have optimal solutions and share the same optimal object objective value. 
    
    Then, based on assumption \eqref{ass_LP} that $|\mathcal{I}^{*}|=|\mathcal{J}^{*}|$, we have 
    \begin{align*}
        |\mathcal{I}^{*}|+|\mathcal{J}'|&=|\mathcal{J}^{*}|+|\mathcal{J}'|=d,\\
        |\mathcal{I}^{*}|+|\mathcal{I}'|&=|\mathcal{J}^{*}|+|\mathcal{I}'|=m.
    \end{align*}
\end{proof}

\subsection{Proof of Proposition \ref{prop_upper_bound}}

\begin{proof}
From $\text{OPT}_{\text{LP}}\geq\text{OPT}$, we know
    \begin{align}
        \label{reg_OPTLP}
        \text{Regret}_T^{\pi}(\mathcal{P}, \bm{B}) 
        &\coloneqq \text{OPT} - \E\left[\sum_{t=1}^{\tau} r_t\right]\\
        &\leq
        \text{OPT}_{\text{LP}}- \E\left[\sum_{t=1}^{\tau} r_t\right]\nonumber
    \end{align}

The idea is to represent $\E\left[\sum_{t=1}^{\tau} r_t\right]$ with reduced cost and number of plays for the arms. We have
    \begin{align*}
        \mathbb{E}\left[\sum\limits_{t=1}^{\tau}r_t\right]
        &=
        \mathbb{E}\left[\sum\limits_{t=1}^{\tau}r_{i_t}\right]\\
        &=
        \sum\limits_{i\in[m]}\mu_i\mathbb{E}\left[n_i(\tau)\right]\\
        &=
        \sum\limits_{i\in[m]}(\bm{c}_i^{\top}\bm{y}^*-\Delta_i)\mathbb{E}\left[n_i(\tau)\right]\\
        &=
        \mathbb{E}\left[\sum\limits_{t=1}^{\tau}\bm{c}_{t}\right]^{\top}\bm{y}^*
        -
        \sum\limits_{i\in\mathcal{I}'}\Delta_i\mathbb{E}\left[n_i(\tau)\right],
    \end{align*}    
where $i_t$ denotes the arm chosen to play at time $t$. In above, the second line comes from the independence between stopping time and the reward generated from a fixed arm. The third line comes from the definition of the reduced cost $\Delta_i$ for all $i\in[d]$ that $\Delta_i=\bm{c}_i^{\top}\bm{y}^*-\mu_i$, and the last line can be obtained by the same proof of the first three lines. Thus, plugging the equations above in the inequality \eqref{reg_OPTLP} and rearranging terms, we have
    \begin{align*}
        \text{Regret}_T^{\pi}(\mathcal{P}, \bm{B}) 
        &\leq
        \text{OPT}_{\text{LP}}- \E\left[\sum_{t=1}^{\tau} r_t\right]\\
        =
        &\mathbb{E}\left[\bm{B}-\sum\limits_{t=1}^{\tau}\bm{c}_{t}\right]^{\top}\bm{y}^*
        +
        \sum\limits_{i\in\mathcal{I}'}\Delta_i\mathbb{E}\left[n_i(\tau)\right]\\
        =
        &\mathbb{E}\left[\bm{B}^{(\tau)}\right]^{\top}\bm{y}^*
        +
        \sum\limits_{i\in\mathcal{I}'}\Delta_i\mathbb{E}\left[n_i(\tau)\right].
    \end{align*}
Thus we complete the proof.
\end{proof}

\subsection{Proof of Proposition \ref{OPTij}}
\begin{proof}
It is sufficient to show that 
\begin{align}
        \text{OPT}_i<\text{OPT}_{\text{LP}} \Leftrightarrow i\in\mathcal{I}^*, \label{OPTi<OPT} \\
        \text{OPT}_j<\text{OPT}_{\text{LP}} \Leftrightarrow i\in\mathcal{J}^*. \label{OPTj<OPT}
    \end{align}
 
Here we prove  \eqref{OPTi<OPT} and the proof of \eqref{OPTj<OPT} follows the same analysis. By definition, an feasible solution to LP \eqref{opt_i} is also feasible to LP \eqref{primalLP}. So,
$$\text{OPT}_i\leq \text{OPT}_{\text{LP}}.$$
If $\text{OPT}_{\text{LP}}=\text{OPT}_i$ for $i\in\mathcal{I}^*$, the optimal solution of LP \eqref{opt_i} is also a optimal solution to \eqref{primalLP}. Denote that solution as $\tilde{\bm{x}}^*$. We have both of $\tilde{\bm{x}}^*$ and $\bm{x}^*$ are optimal solutions of LP \eqref{primalLP}. From the uniqueness in Assumption \ref{ass_LP}, we know $\tilde{\bm{x}}^*=\bm{x}^*$ and hence $x^*_i\not=0$ and $\tilde{x}^*_i\not=0$ for $i \in \mathcal{I}^*$. However, this contradicts with constraint of \eqref{opt_i} where $x_i=0$. 
\end{proof}

\subsection{Proof of Lemma \ref{paramEst}}
\begin{proof}
    The proof follows a standard analysis of the tail bounds for the estimators. Specifically, we first consider the complement of these events and then apply a union bound to complete the proof. The analysis first concerns the case when there is no projection and then extends the results to the case with projection.
    
    For each $i\in[m],$ we have
    \begin{align}
       & \mathbb{P}\left(
            \bigcup_{t\in[T]}
            \left\{
                \frac{1}{n_i(t)}\left|\sum\limits_{s=1}^{n_i(t)}r_{i,s}-n_i(t)\mu_i\right|\geq \sqrt{\frac{2\log T}{n_i(t)}}
            \right\}
        \right) \nonumber \\
        \leq&
        \sum\limits_{t=1}^{T}
        \mathbb{P}\left(
                \frac{1}{n_i(t)}\left|\sum\limits_{s=1}^{n_i(t)}r_{i,s}-n_i(t)\mu_i\right|\geq \sqrt{\frac{2\log T}{n_i(t)}}
        \right)\nonumber \\
        \leq&
        \sum\limits_{t=1}^{T} \sum\limits_{k=1}^{t}
        \mathbb{P}\left(
                \frac{1}{k}\left|\sum\limits_{s=1}^{k}r_{i,s}-k\mu_i\right|\geq \sqrt{\frac{2\log T}{k}}
        \right)\mathbb{P}\left(n_i(t)=k\right)\nonumber \\
        \leq&
        \sum\limits_{t=1}^{T} \sum\limits_{k=1}^{t}
        \mathbb{P}\left(
                \frac{1}{k}\left|\sum\limits_{s=1}^{k}r_{i,s}-k\mu_i\right|\geq \sqrt{\frac{2\log T}{k}}
        \right)\nonumber \\
        \leq&  \sum\limits_{t=1}^{T} \sum\limits_{k=1}^{t} \frac{2}{T^4}
        \leq \frac{2}{T^2}. \label{concentration_on_r}
    \end{align}
    where $r_{i,s}$ denotes the reward of the $s$-th play for the $i$-th arm. The first inequality is obtained by the union probability bound, the second one follows from Bayes rule, the third one is obtained by $\mathbb{P}\left[n_i(t)=k\right]\leq1$, and the last one is obtained by an application of the Hoeffding's inequality with $r_{i,s}\in[0,1]$ and $\mathbb{E}[r_{i,s}]=\mu_i$ for $i\in[m]$ and $s\in[T]$.  
    
Then, recall that 
\begin{align*}
    \mu_i^L(t)& \coloneqq  proj_{[0,1]}\left(\hat{\mu}_{i}(t)-\sqrt{\frac{2\log T}{n_i(t)}}\right), \\
    \mu_i^U(t)&\coloneqq  proj_{[0,1]}\left(\hat{\mu}_{i}(t)+\sqrt{\frac{2\log T}{n_i(t)}}\right),
\end{align*}
and $\mu_i\in[0,1]$, following the definitions of the sets, we can have 
\begin{align*}
   \mathbb{P}\left(
        \bigcup_{t\in[T]}
\left\{\mu_i\notin\left(\mu_i^{L}(t),\mu_i^{U}(t)\right)\right\}
    \right)
    &=
        \mathbb{P}\left(
        \bigcup_{t\in[T]}
        \left\{
            \frac{1}{n_i(t)}\left|\sum\limits_{s=1}^{n_i(t)}r_{i,s}-n_i(t)\mu_i\right|\geq \sqrt{\frac{2\log T}{n_i(t)}}
        \right\}
    \right).
\end{align*}
Then we can apply \eqref{concentration_on_r} for the last line in above.
\begin{align*}
&\ \ \ \ \mathbb{P}\left(
            \bigcup_{t\in[T]}
    \left\{\mu_i\notin\left(\mu_i^{L}(t),\mu_i^{U}(t)\right)\right\}
\right) \leq\frac{2}{T^2}.
\end{align*}
Note the proof is only relied on the property that the reward is bounded in $[0,1]$.  

Next, following the exact same analysis, we have
    \begin{align*}
        &\ \ \ \ \mathbb{P}\left(
            \bigcup_{t\in[T]}
    \left\{c_{ji}\notin\left(C_{ji}^{L}(t),C_{ji}^{U}(t)\right)\right\}
        \right)\leq\frac{2}{T^2}
    \end{align*}
    also holds for each $j\in[d],i\in[m]$. Then, we complete the proof by taking an union bound with respect to the arm index $i\in[m]$ and the constraint index $j\in[d].$
\end{proof}

\subsection{Proof of Lemma \ref{optl_u}}
\begin{proof}
It is sufficient to show that 
$$\text{OPT}_{\text{LP}}^L \le \text{OPT}_{\text{LP}} \le \text{OPT}_{\text{LP}}^U$$
holds under the event in Lemma \ref{paramEst}. Then, applying Lemma \ref{paramEst}, we can have that the inequality above holds with probability no less than $1-\frac{4md}{T^2}.$
    
Recall that $\text{OPT}^{L}_{\text{LP}}(t)$ corresponds to the optimal objective value of the following LP \eqref{UCB_LP} and denote ${\bm{x}}^{L}(t)$ as its optimal solution. Since $\hat{{\bm{C}}}^{U}(t)\geq\bm{C}$ holds element-wise under the event in Lemma \ref{paramEst}, we have 
    $$
        \bm{C}{\bm{x}}^{L}(t)\leq{{\bm{C}}}^{U}(t){\bm{x}}^{L}(t)\leq \bm{B},
    $$
    where those inequalities hold element-wise. Thus, we have ${\bm{x}}^{L}(t)$ is a feasible solution to the LP \eqref{primalLP}, and
    \begin{align}
        \text{OPT}^{L}_{\text{LP}}(t)
        &=
        \sum\limits_{i=1}^{m}{{\mu}}^{L}_i(t) {x}^L_i(t)\nonumber\\
        &\leq
         \sum\limits_{i=1}^{m}{{\mu}}_i(t) {x}^L_i(t)\\
         &\leq
         \text{OPT}_{\text{LP}}, \nonumber
    \end{align}
    where the first line comes the definition of ${\bm{x}}^{L}(t)$, the second line comes from ${\hat{\bm{\mu}}}^L(t)\leq\bm{\mu}(t)$ (under the event in Lemma \ref{paramEst}), and the third line comes from that ${\bm{x}}^L(t)$ is feasible to LP \eqref{primalLP} and the definition of $\text{OPT}_{\text{LP}}$. Thus, $\text{OPT}^{L}_{\text{LP}}(t)$ is an underestimate of $\text{OPT}_{\text{LP}}$ for all $t\in[T]$.
    
With a similar analysis, we can show the part for $\text{OPT}^{U}_{\text{LP}}(t)$.
\end{proof}    

\renewcommand{\thesubsection}{B\arabic{subsection}}

\section{Analysis of Phase I of Algorithm \ref{alg_BwK}}

In this section, we provide the analysis of Phase I of Algorithm \ref{alg_BwK} and prove Proposition \ref{PhaseI_UB}. In the proof, we will use some concentration results for the LCB/UCB estimators for the LP, and these concentration results will be summarized and proved in Lemma \ref{lem:conv} following the proof of Proposition \ref{PhaseI_UB}.

\subsection{Proof of Proposition \ref{PhaseI_UB}}
\begin{proof}
In this proof, we show the result under the high-probability event in Lemma \ref{paramEst}.

Recall that an arm $i$ will be added to $\hat{\mathcal{I}}^*$ or resource $j$ will be added to $\hat{\mathcal{J}}'$ when 
$$
    \text{OPT}^{U}_i(t)<\text{OPT}^{L}_{\text{LP}}(t)\text{ or }\text{OPT}^{L}_j(t)>\text{OPT}_{\text{LP}}^{U}(t).
$$
So, our proof will be focusing on discussing the relation between $\text{OPT}^{U}_i(t)$ and $\text{OPT}^{L}_{\text{LP}}(t)$, and the relation between $\text{OPT}^{U}_j(t)$ and $\text{OPT}^{L}_{\text{LP}}(t)$. Our proof can be divided into two parts:
\begin{itemize}
    \item First, we prove that $\hat{\mathcal{I}}^*$ and $\hat{\mathcal{J}}'$ will not include a non-optimal arm $i\in\mathcal{I}'$ and binding resource $j\in\mathcal{J}^*$. Equivalently, we only need to show
    $$\text{OPT}^{U}_i(t)\geq \text{OPT}^{L}_{\text{LP}}(t), \text{OPT}^{U}_j(t)\geq \text{OPT}_{\text{LP}}^{L}(t)$$
    for $i\in\mathcal{I}'$ and $j\in\mathcal{J}^*$ and all $t\in[T]$. 
    \item Second, Phase I of Algorithm \ref{alg_BwK} will terminate within $\left(2+\frac{1}{b}\right)^2\frac{72m\log T}{\delta^2}$ steps. Equivalently, we only need to show $$\text{OPT}^{U}_i(t)<\text{OPT}^{L}_{\text{LP}}(t), \text{OPT}^{U}_j(t)<\text{OPT}^{L}_{\text{LP}}(t)$$ for $i\in \mathcal{I}^*$ and $j\in\mathcal{J}'$ when the time  $$t\geq\left(2+\frac{1}{b}\right)^2\frac{72m\log T}{\delta^2}.$$ 
\end{itemize}  

For the first part, recall that $\text{OPT}^{U}_i$ is the optimal objective value of the LP \eqref{opt_i}.
For $i\in[m]$, with a similar proof as Lemma \ref{optl_u}, we can show that
$$
    \text{OPT}^{U}_i(t)\geq \text{OPT}_i
$$
and from Lemma \ref{optl_u},
$$\text{OPT}^{L}_{\text{LP}}(t)\leq \text{OPT}_{\text{LP}}$$
for all $t\in[T]$. For $i\in\mathcal{I}'$, we have $\text{OPT}_i=\text{OPT}_{\text{LP}}$ from Proposition \ref{OPTij}. So,
$$\text{OPT}^{U}_i(t)\geq \text{OPT}^{L}_{\text{LP}}(t).$$
With a similar argument, 
$$\text{OPT}^{U}_j(t)\geq \text{OPT}_{\text{LP}},\text{ for all $j\in\mathcal{J}^*$}.$$
Therefore, we complete the first part. The implication is that $\hat{\mathcal{I}}^*\subset{\mathcal{I}}^*$ and $\hat{\mathcal{J}}'\subset{\mathcal{J}}'$.

Recall that the algorithm terminates when $|\hat{\mathcal{I}}^*|+|\hat{\mathcal{J}}'|\geq d$ and that from Lemma \ref{I_J_set}, $|{\mathcal{I}}^*|+|{\mathcal{J}}'|=d$. So, we have the algorithm terminates if and only if $\hat{\mathcal{I}}^*={\mathcal{I}}^*$ and ${\hat{\mathcal{J}}'}={\mathcal{J}}'$. That is, the algorithm stops only when it finds all optimal arms and all non-binding resources. 

Now, we move to the second part and prove that $\hat{\mathcal{I}}^*={\mathcal{I}}^*$ and ${\hat{\mathcal{J}}'}={\mathcal{J}}'$ within $$\left(2+\frac{1}{b}\right)^2\frac{72m\log T}{\delta^2}$$ steps. 

For time $t$ that is a multiple of $m$ and that 
    $$t\geq\left(2+\frac{1}{b}\right)^2\frac{72m\log T}{\delta^2},$$ we have 
    $$
        3\sqrt{\frac{2\log T}{n_i(t)}}\leq\frac{b}{2}
    $$
    i.e., the condition \eqref{opt_conv_cond} in the statement of Lemma \ref{lem:conv} holds, where $n_i(t)=\frac{t}{m}$ in this case. Then, by Lemma \ref{lem:conv}, we have
\begin{align*}
    \frac{1}{T}|\text{OPT}^{L}(t)-\text{OPT}_{\text{LP}}|&<\frac{\delta}{2},\\
    \frac{1}{T}|\text{OPT}^{U}_i(t)-\text{OPT}_i|&<\frac{\delta}{2},\\
    \frac{1}{T}|\text{OPT}^{U}_j(t)-\text{OPT}_j|&<\frac{\delta}{2}.
\end{align*}    
Thus, within at most $\left(2+\frac{1}{b}\right)^2\frac{72m\log T}{\delta^2}$ steps, we have
\begin{align*}
    \text{OPT}^{L}(t)-\text{OPT}^{U}_i
    >
    \text{OPT}_{\text{LP}}-\frac{\delta}{2}
    -
    \text{OPT}_i + \frac{\delta}{2}
    \geq
    0,\\
    \text{OPT}^{L}(t)-\text{OPT}^{U}_j
    >
    \text{OPT}_{\text{LP}}-\frac{\delta}{2}
    -
    \text{OPT}_j + \frac{\delta}{2}
    \geq
    0,
\end{align*}
for all $i\in\mathcal{I}^*$ and $j\in\mathcal{J}'$ since $\text{OPT}-\text{OPT}_i\geq\delta$ and $\text{OPT}_j-\text{OPT}_{\text{LP}}\geq\delta$. Thus, all optimal arms and non-binding resources will be detected with at most $\left(2+\frac{1}{b}\right)^2\frac{72m\log T}{\delta^2}$ steps, equivalently, with each arm played at most $\left(2+\frac{1}{b}\right)^2\frac{72\log T}{\delta^2}$ times. 

\end{proof}

\subsection{Concentration of LP UCB/LCB Estimators}

\begin{lemma}
\label{lem:conv}
In Phase I of Algorithm \ref{alg_BwK}, for $t\in[T]$ that is the multiple of $m$ and $t\le B$, the following inequalities hold for all $i\in[m]$ and $j\in[d]$,
% \frac{1}{T}|\text{OPT}^{U}(t)-\text{OPT}_{\text{LP}}|&<3\left(1+\frac{1}{b}\right)\sqrt{\frac{2\log T}{n_i(t)}},\label{optu_converge}\\
    \begin{align}
        \frac{1}{T}|\text{OPT}^{U}_i(t)-\text{OPT}_i|&<3\left(1+\frac{1}{b}\right)\sqrt{\frac{2\log T}{n_i(t)}},\label{opti_converge}\\
        \frac{1}{T}|\text{OPT}^{U}_j(t)-\text{OPT}_j|&<3\left(2+\frac{1}{b}\right)\sqrt{\frac{2\log T}{n_i(t)}}.\label{optj_converge}
    \end{align}  
    Moreover, if time $t$ also satisfies 
    \begin{align}
    \label{opt_conv_cond}
    3\ \sqrt{\frac{2\log T}{n_i(t)}}\leq\frac{b}{2},
    \end{align}
    the following inequality also holds
    \begin{align}
        \frac{1}{T}|\text{OPT}^{L}(t)-\text{OPT}_{\text{LP}}|&<3\left(1+\frac{1}{b}\right)\sqrt{\frac{2\log T}{n_i(t)}}.\label{optl_converge}
    \end{align}
By the design of Phase I of Algorithm \ref{alg_BwK}, all the $n_i(t)$ in this lemma (the statement and the proof) equals to $t/m.$
\end{lemma}

\begin{proof}
We prove inequalities in the statement one by one. We note that the condition $t\leq B$ is to provide a guarantee that the algorithm does not exhaust the resource before the termination of Phase I of the algorithm.

\paragraph{Proof of \eqref{opti_converge}:}

For \eqref{opti_converge}, we only need to show that 
 $$\frac{1}{T}(\text{OPT}_i^{U}-\text{OPT}_i)< 3\left(1+\frac{1}{b}\right)\sqrt{\frac{2\log T}{n_i(t)}}$$
since for the other direction, we can use the inequality $\text{OPT}^{U}_i(t)\geq \text{OPT}_i$ with the same proof of Lemma \ref{optl_u}. 

We introduce the following LP \eqref{LP:mid} to relate $\text{OPT}_i^{U}(t)$ with $\text{OPT}_i$.
        \begin{align}
            \label{LP:mid}
                \max_{\bm{x}} \ \   & \bm{\mu}^\top \bm{x} +3\sqrt{\frac{2\log T}{n_i(t)}}T\nonumber\\
                \text{s.t.}\ \ &{\bm{C}} \bm{x} \leq \left(1+3\sqrt{\frac{2\log T}{n_i(t)b^2}}\right)\bm{B} \\
                & \ \ \ \ x_{i} = 0,\ \bm{x} \ge 0. \nonumber
        \end{align}  
Compared to \eqref{opt_i} which defines $\text{OPT}_i$, the LP \eqref{LP:mid} scales the right-hand-side resource capacity with a factor and also appends an additional term to the objective function. It is easy to see that the optimal objective value of the LP \eqref{LP:mid} is 
        \begin{align}
        \label{LPmidOPT}
            \left(1+3\sqrt{\frac{2\log T}{n_i(t)b^2}}\right)\text{OPT}_i+3\sqrt{\frac{2\log T}{n_i(t)}}T.
        \end{align}

Denote the optimal solution of LP \eqref{opt_i} as $\bm{x}^{U}(t)$. First, we show that $\bm{x}^{U}(t)$
is a feasible solution to \eqref{LP:mid}, and then, with the feasibility, we show that \eqref{LPmidOPT} provides an upper bound for OPT$_i^{U}$ and thus establish the relation between OPT$_i^{U}$ and OPT$_i.$

To see the feasibility, note that the $j$-th entry of the resource consumption 
    \begin{align*}
            (\bm{C}\bm{x}^{U}(t))_j
            &=
            ({\bm{C}}^{L}(t)\bm{x}^{U}(t))_j+\left(\left(\bm{C}-{\bm{C}}^{L}(t)\right)\bm{x}^{U}(t)\right)_j\\
            &\leq
            B+\left\|\left(\bm{C}-{\bm{C}}^{L}(t)\right)_j\right\|_{\infty}\|\bm{x}^{U}(t)\|_{1}\\
            &\leq
            B+3\sqrt{\frac{2\log T}{n_i(t)}}T,
        \end{align*}
where the last line comes from the concentration event in Lemma \ref{paramEst} and the fact that $\|\bm{x}^{U}(t)\|_1 \le T.$

With this feasibility, we have 
\begin{align}
            \text{OPT}^{U}_i(t)\nonumber
            =
            &({\bm{\mu}}^{U}(t))^{\top} \bm{x}^{U}(t)\nonumber\\
            =
            &\bm{\mu}^{\top} \bm{x}^{U}(t) + \left({\bm{\mu}}^U(t)-\bm{\mu}\right)^{\top} \bm{x}^{U}(t)\nonumber\\
            \leq
            &\bm{\mu}^{\top} \bm{x}^{U}(t) + \left\|{\bm{\mu}}^{U}(t)-\bm{\mu}\right\|_{\infty} \left\|\bm{x}^{U}(t)\right\|_{1} \nonumber \\
            <
            &\bm{\mu}^{\top} \bm{x}^{U}(t) + 3\sqrt{\frac{2\log T}{n_i(t)}}T\nonumber\\
            \leq
            &\left(1+3\sqrt{\frac{2\log T}{n_i(t)b^2}}\right)\text{OPT}_i+\sqrt{\frac{2\log T}{n_i(t)}}T. \label{upb_LPmid}
        \end{align}
Here the second from last line comes from a similar argument as the analysis of feasibility. It is based on the concentration event in Lemma \ref{paramEst} and the fact $\|\bm{x}^{U}(t)\|_1 \le T.$ The last line plugs in the upper bound \eqref{LPmidOPT} to LP \eqref{LP:mid} and it comes from the fact that $\bm{x}^{U}(t)$ is a feasible solution of LP \eqref{LP:mid}. By rearranging terms in \eqref{upb_LPmid}, we complete the proof of \eqref{opti_converge}.

\paragraph{Proof of \eqref{optj_converge}:}

Just like the proof of \eqref{opti_converge}, we only need to show one side of the inequality, that is 
$$ \frac{1}{T}(\text{OPT}_j^{U}-\text{OPT}_j)\leq 3\left(2+\frac{1}{b}\right)\sqrt{\frac{2\log T}{n_i(t)}}.$$
Denote $\tilde{\bm{y}}$ as the optimal solution of the LP \eqref{opt_j}, which is the LP corresponding to $\text{OPT}_j$. Define
$$\tilde{\bm{y}}'=\tilde{\bm{y}}+3\left(2b+1\right)\sqrt{\frac{2\log T}{n_i(t)}}(1,0,...,0)^\top.$$
We first show that $\tilde{\bm{y}}'$ is feasible the LP corresponding to $\text{OPT}_j^{U}(t)$ in Algorithm \ref{alg_BwK}. To see the feasibility, we have 
    \begin{align*}
        \left(\bm{C}^{L}(t)\right)^{\top}\tilde{\bm{y}}'\nonumber
        =
        &\left(\bm{C}^{L}(t)\right)^{\top}\tilde{\bm{y}}+3\left(2b+1\right)\sqrt{\frac{2\log T}{n_i(t)}}\left(\bm{C}^{L}(t)\right)^{\top}(1,0,...,0)^\top\\
        =
        &\bm{C}^{\top}\tilde{\bm{y}}+(\bm{C}^{L}(t)-\bm{C})^{\top}\tilde{\bm{y}}
        +
        3\left(2b+1\right)\sqrt{\frac{2\log T}{n_i(t)b^2}}\bm{1}\\
        \geq
        &\bm{\mu}+(\bm{C}_{j,\cdot})^{\top}
        -
        \max\limits_{j\in[d]}\left\| \bm{c}^{L}_j(t)-\bm{c}_j\right\|_{\infty}(\bm{1}^{\top}\tilde{\bm{y}})\bm{1}+
        3\left(2+\frac{1}{b}\right)\sqrt{\frac{2\log T}{n_i(t)}}\bm{1}\\
        \geq
        &\bm{\mu}+(\bm{C}_{j,\cdot})^\top+6\sqrt{\frac{2\log T}{n_i(t)}}\bm{1}\\
        =
        &\bm{\mu}^{U}(t)+\left(\bm{C}_{j,\cdot}^{U}(t)\right)^{\top}.
    \end{align*}
Here the first two lines come from rearranging terms. The third line comes from the feasibility of $\tilde{\bm{y}}$. The fourth line comes from the concentration event in Lemma \ref{paramEst} and the fact that $\bm{1}^\top \tilde{\bm{y}} \le \frac{T}{B}.$ The last line comes from the definition of UCB estimators. 

With the feasibility of $\tilde{\bm{y}}',$ we have 
    \begin{align*}
        \text{OPT}^{U}_j
        &\leq
        \bm{B}^{\top}\tilde{\bm{y}}'\\
        &\leq
        \bm{B}^{\top}\tilde{\bm{y}} + 9\sqrt{\frac{2\log T}{n_i(t)}}T\\
        &=
        \text{OPT}_j + 9\sqrt{\frac{2\log T}{n_i(t)}}\\
        &<
        \text{OPT}_j + 3\left(2+\frac{1}{b}\right)\sqrt{\frac{2\log T}{n_i(t)}},
    \end{align*}
where the first comes from the feasibility of $\tilde{\bm{y}}'$, the second and fourth line uses the fact that $b<1,$ and the third line comes from the definition of OPT$_j.$ By rearranging terms, we complete the proof of \eqref{optj_converge}.

\paragraph{Proof of inequality \eqref{optl_converge}:}
The idea of the proof is to first show that 
$$\left(1-3\sqrt{\frac{2\log T}{n_i(t)b^2}}\right)\bm{x}^* $$ is a feasible solution to the LP corresponding to OPT$_{\text{LP}}^L(t)$ under condition \eqref{opt_conv_cond} and then analyze its corresponding objective value.
    
First, to see its feasibility, the condition \eqref{opt_conv_cond} ensures that 
$$\left(1-3\sqrt{\frac{2\log T}{n_i(t)b^2}}\right)\bm{x}^{*}\geq \bm{0}.$$
On the other hand, for the constraint $j\in[d],$
\begin{align*}
            B-\left(1-3\sqrt{\frac{2\log T}{n_i(t)b^2}}\right)({\bm{C}}^{U}(t)\bm{x}^*)_j\nonumber
            =
            &B-({\bm{C}}^{U}(t)\bm{x}^*)_j+3\sqrt{\frac{2\log T}{n_i(t)b^2}}({\bm{C}}^{U}(t)\bm{x}^*)_j\nonumber\\
            = 
            &B-(\bm{C}\bm{x}^*)_j+(({\bm{C}}^{U}(t)-\bm{C})\bm{x}^*)_j
            +
            3\sqrt{\frac{2\log T}{n_i(t)b^2}}(\bm{C}\bm{x}^*)_j\nonumber\\
            \geq
            &B-(\bm{C}\bm{x}^*)_j-\|({\bm{C}}^{U}(t)-\bm{C})_{j,\cdot}\|_{\infty}\|\bm{x}^*\|_1\nonumber
            +3\sqrt{\frac{2\log T}{n_i(t)b^2}}(\bm{C}\bm{x}^*)_j\nonumber\nonumber\\
            \geq
          &B-(\bm{C}\bm{x}^*)_j+3\sqrt{\frac{2\log T}{n_i(t)}}\left(
       \frac{1}{b}(\bm{C}\bm{x}^*)_j-T\nonumber
            \right).
        \end{align*}
Here the first and second equality comes from rearranging terms and the last two line follow a similar analysis as before.

By rearranging terms in the last line of above,
        \begin{align*}
            B-\left(1-3\sqrt{\frac{2\log T}{n_i(t)b^2}}\right)({\bm{C}}^{U}(t)\bm{x}^*)_j
            \geq
            &B-(\bm{C}\bm{x}^*)_j+3\sqrt{\frac{2\log T}{n_i(t)}}\left(
                \frac{1}{b}(\bm{C}\bm{x}^*)_j-T
            \right)\\
            =
            &-\left(1-3\sqrt{\frac{2\log T}{n_i(t)b^2}}\right)(\bm{C}\bm{x}^*)_j+B-
                3\sqrt{\frac{2\log T}{n_i(t)}}T\\
            \geq
                &0.
        \end{align*}
Here the first from last line comes from rearranging terms and the last line comes from the feasibility of $\bm{x}^*.$ In this way, we obtain the feasibility part. 

Given the feasibility, we can prove the inequality \eqref{optl_converge} with a similar roadmap as before
        \begin{align*}
            \text{OPT}^{L}_{\text{LP}}(t)
            &\geq
            \left(1-3\sqrt{\frac{2\log T}{n_i(t)b^2}}\right)({{\bm{\mu}}}^{L}(t))^{\top}\bm{x}^*\\
            &=
            \bm{\mu}^{\top}\bm{x}^* + \left({{\bm{\mu}}}^{L}(t)-\bm{\mu}\right)^{\top}\bm{x}^*-
            3\sqrt{\frac{2\log T}{n_i(t)b^2}}({{\bm{\mu}}}^{L}(t))^{\top}\bm{x}^*\\
            &\geq
          \text{OPT}_{\text{LP}}-\left\|{{\bm{\mu}}}^{L}(t)-\bm{\mu}\right\|_{\infty}\|\bm{x}^*\|_{1}-     
            3\sqrt{\frac{2\log T}{n_i(t)b^2}}({{\bm{\mu}}}^{L}(t))^{\top}\bm{x}^*\\
            &\ge     \text{OPT}_{\text{LP}}-3\left(1+\frac{1}{b}\right)\sqrt{\frac{2\log T}{n_i(t)}}T.
        \end{align*}
Thus we complete the proof of inequality \eqref{optl_converge}. 
\end{proof}

\section{Analysis of Phase II of Algorithm \ref{alg_BwK}}
\renewcommand{\thesubsection}{C\arabic{subsection}}

In this section, we analyze Phase II of Algorithm \ref{alg_BwK} and prove Proposition \ref{prop_B_tau}. To simplify the presentation, we first introduce a notion of ``warm start'' and prove Proposition \ref{prop_B_tau} under this warm start condition and also, the assumption that all the constraints are binding. Indeed, recall that the upper bound in Proposition \ref{prop_upper_bound} only involves the remaining resource for binding constraints. The assumption is only aimed to simplify the presentation but will not change the natural of our analysis. Later in Section \ref{sectionWarmStart} and Section \ref{ana_non_binding}, we discuss how the warm start condition is achieved and what if there are non-binding constraints, respectively. 

\subsection{Proof of Proposition \ref{prop_B_tau} under Warm Start Assumption}

\label{MainProof}
We first define the following quantity
 $$ \theta=\min\left\{\frac{\min\{1,\sigma^2\}\min\{\chi,\delta\}}{12\min\{m^2,d^2\}},\left(2+\frac{1}{b}\right)^{-2}\cdot\frac{\delta}{5}
 \right\}.$$
 
In the following, we will use $\theta$ to define a threshold for parameter estimation. As we will see later, the threshold is mainly to ensure that the estimated constraint matrix $\bm{C}^{L}(t)$ is well-positioned (with a minimum singular value of $\sigma/2$).

\begin{assumption}[Warm Start]
We make the following assumption on the parameter estimates
\begin{itemize}
    \item[(a)] $\bm{\mu}^{U}(t)\geq\bm{\mu},\bm{C}^{L}(t)\leq\bm{C}$ holds element-wise for all $t\in[T]$.
    \item[(b)] The following condition holds for all $i\in\mathcal{I}^*$ and $t\in[T]$
        \begin{align}
        \label{warm_start}
            &\|{\bm{c}}^{L}_{i}(t)-\bm{c}_i\|_{\infty}\leq\theta.
        \end{align}
    \end{itemize}
    \label{warmStart}
\end{assumption}

Here Part (a) of the Assumption restricts to the concentration event in Lemma \ref{paramEst}. Part (b) is to ensure that the estimated LP with constraint matrix $\bm{C}^L(t)$ is well-positioned. In the following Section \ref{sectionWarmStart}, we will show that the assumption can be met automatically in Algorithm \ref{alg_BwK} and we discuss the number of time periods it takes to satisfy this assumption.

\begin{assumption}[All optimal arms and binding constraints]
We assume that all the arm are optimal and all the constraints are binding, i.e., 
$$\mathcal{I}^*=[m], \mathcal{J}^*=d.$$
\label{assume_all_binding}
\end{assumption}

In fact, Assumption \ref{assume_all_binding} is not as strong as it seems to be. First, in Algorithm \ref{alg_BwK}, we have shown that Phase I identifies optimal arms and binding constraints with high probability in Proposition \ref{PhaseI_UB}. Then, for the LP \eqref{adaptLP} solved in Phase II of the algorithm, it restricts the arm play to the optimal arm as if all the arms are all optimal ones. Second, from Proposition \ref{prop_upper_bound}, we know the generic upper bound only involves the remaining resources for the binding constraints, so this assumption allows us to focus on the binding constraints. To make the analysis go through, we only need to show that the non-binding resources are not exhausted before any binding resource is exhausted. Intuitively, this will always happen because the non-binding resources have a slackness. Later in Section \ref{ana_non_binding}, we will return to analyze this case of non-binding resources. Now, we prove Proposition \ref{prop_B_tau} under this two additional assumptions.

\paragraph{Proof of Proposition \ref{prop_B_tau} under Assumption \ref{warmStart} and \ref{assume_all_binding}}

We show that under Assumption \ref{warmStart} and \ref{assume_all_binding}, there exists a constant $\underline{T}$ depending on $m$, $d$, $\sigma$, $\chi$ and $\delta$ such that, for any $T\geq\underline{T}$,  
the left-over of each binding resource (at time $\tau$) $j\in\mathcal{J}^*$ satisfies
    $$\E\left[B_{j}^{(\tau)}\right] = O\left(\frac{d^3}{b\min\{\chi^2, \delta^2\}\min\{1,\sigma^2\}}\right).$$

\begin{proof}
Our proof basically formalize the idea described in Section 5.2. Recall that the average remaining resource
$$\bm{b}^{(t)}\coloneqq\frac{\bm{B}^{(t)}}{T-t}$$
    where $\bm{b}^{(t)}=\left(b_{1}^{(t)},...,b_{d}^{(t)}\right)^\top$ for $t=0,...,T-1.$ Define 
$$\mathcal{Z} \coloneqq  [b-\epsilon,b+\epsilon]^d \subset \mathbb{R}^d$$
where $\epsilon>0$ will be specified later. Ideally, the algorithm consumes approximately $\bm{b}$ resource per time period so that the average remaining resource $\bm{b}^{(t)}$ will always stay in $\mathcal{Z}$. This motivates the definition of a stopping time that represents the first time that $\bm{b}^{(t)}$ escapes from $\mathcal{Z}$.
$$\tau' \coloneqq \min\{t:\bm{b}^{(t)} \notin \mathcal{Z}\}.$$
Recall that the stopping time $\tau$ is the termination of the procedure. Thus, as long as $\epsilon<b$, we have $\tau'\leq \tau$ since no resource is used up at time $\tau'-1$. 
Furthermore, we define the following event, for each $t\in[T]$,
\begin{align*}
    \mathcal{E}_t \coloneqq  &\left\{\bm{b}^{(s)}\in\mathcal{Z}\text{ for each $s\in[t-1]$} \ \text{ and }\ \sup_{\bm{b}'\in \mathcal{Z}}\left\|\mathbb{E}[\bm{c}_{i_{t}}(\bm{b}')|\mathcal{H}_{t-1}]-\bm{b}'\right\|_\infty > \epsilon_t \right\}
\end{align*}
where $\mathcal{H}_{t}=\{(\bm{r}_s, \bm{C}_s)\}_{s=1}^t$ is the filtration up to time $t$, $i_t$ is the arm played by Algorithm \ref{alg_BwK} in time period $t$ during Phase II, and $\bm{c}_{i_t}(\bm{b}')$ is the expected resource consumption of the $i_t$-th arm. The threshold parameters $\epsilon_t$ will be specified later. We remark that in Algorithm \ref{alg_BwK}, the choice of $i_t$ (and its distribution) is dependent on the remaining resource $\bm{b}^{(t-1)}$. Here we represent $\bm{b}^{(t-1)}$ with $\bm{b}'$, and thus the second term captures the maximal ``bias'' of resource consumption when the average remaining resource $\bm{b}^{(t-1)}=\bm{b}'$, where the maximum is taken with respect to all $\bm{b}'\in \mathcal{Z}.$ 

At time $t$, if $\bm{b}^{(t-1)}=\bm{b}'$ for some $\bm{b}'\in \mathcal{Z},$ there are in total $(n-t+1)\bm{b}'$ remaining resources and $n-t+1$ remaining time periods. Ideally, we hope that the resource is consumed for exactly $\bm{b}'$ in each of the remaining time period. In this sense, the event $\mathcal{E}_t$ describes that the average resource level $\bm{b}_s$ stays in $\mathcal{Z}$ for the first $t-1$ time periods and the consumption bias at time $t$ is greater than $\epsilon_t$.

The following lemma states that with a careful choice of $\epsilon$ and $\epsilon_t$, the event $\mathcal{E}_t$ will happen with a small probability. 
\begin{lemma}
\label{lemma:random_resource_stability}
Let $\alpha = \frac{\min\{1,\sigma\}\min\{\chi,\delta\}b}{d^{\frac{2}{3}}}$. If we choose $\epsilon \coloneqq \frac{\alpha}{5}$ and 
    $$
    \epsilon_t \coloneqq \begin{cases}
    \frac{6}{5},& t \leq  \frac{\alpha T}{19}, \\
    \frac{11b\sqrt{\min\{m^3,d^3\}\log T}}{\sqrt{\min\{\chi\sigma^2,\delta\sigma^2\}t}}, & t > \frac{\alpha T}{19},
    \end{cases}
    $$
    then there exists a constant $\underline{T}$ such that,
    $$
        \mathcal{P}\left(
            \bigcup\limits_{t=1}^{T} \mathcal{E}_t
        \right)
        \leq
        \frac{\min\{m,d\}}{T^3}.
    $$
    for all $T\geq\underline{T}$.
\end{lemma}

The proof of Lemma \ref{lemma:random_resource_stability} will be postponed to the end of this proof. The following lemma considers the complements of the event $\mathcal{E}_t$'s and analyze the stopping time on their complements. Here we use $\mathcal{E}^{\complement}$ to denote the complement of an event $\mathcal{E}.$

\begin{lemma}
    \label{lemma:tilde_tau_control}
 If $\epsilon$, $\bar{\epsilon}$, $\alpha$ and $\epsilon_t$ satisfy that $\epsilon_t=\bar{\epsilon}$ for all $t\leq \alpha T$ and
    \begin{align}
      \frac{\alpha\bar{\epsilon}}{1-\alpha}+\sum\limits^{T-1}_{t=\alpha T+1}\frac{\epsilon_t}{T-t}\leq\frac{2{\epsilon}}{3}, \label{eps_condition}   
    \end{align}
    the following inequality holds
    \begin{align*}
        &\mathbb{P}\left(
            b^{(s)}_j\not\in[b-\epsilon,b+\epsilon] \text{ for some $s\leq t$, }
            \bigcap\limits_{s=1}^{t}\mathcal{E}_s^{\complement}
        \right)
        \leq
        2\exp\left\{-\frac{(T-t+1)\epsilon^2}{18(1+\epsilon)^2}\right\},    
    \end{align*}
    for $t\le T-2$ and $j\in[d]$, 
    and therefore,
    $$
        \mathbb{P}\left(
           \tau'\le t, \bigcap\limits_{s=1}^{t}\mathcal{E}_s^{\complement}
        \right)
        \leq
        2d\exp\left\{-\frac{(T-t+1)\epsilon^2}{18(1+\epsilon)^2}\right\}.
    $$
\end{lemma}

The proof of Lemma \ref{lemma:tilde_tau_control} will also be postponed to the end of this proof.

With Lemma \ref{lemma:random_resource_stability} and \ref{lemma:tilde_tau_control}, we can derive an upper bound on $\E[T-\tau'].$ Specifically,
\begin{align*}
         \mathbb{P}\left(
            \tau'\leq t
         \right)
         =
         &\mathbb{P}\left(
           \tau'\leq t, \bigcap\limits_{s=1}^{t}\mathcal{E}_t^{\complement}
         \right)     
         +
         \mathbb{P}\left(
            \tau'\leq t, \bigcup\limits_{s=1}^{t}\mathcal{E}_t
         \right)         \\
        \leq
         &\mathbb{P}\left(
           \tau'\leq t, \bigcap\limits_{s=1}^{t}\mathcal{E}_t^{\complement}
         \right)     
         +
         \mathbb{P}\left(
            \bigcup\limits_{s=1}^{t}\mathcal{E}_t
         \right)
         \\
         \leq
         &2d\exp\left\{-\frac{\min\{1,\sigma^2\}\min\{\chi^2,\delta^2\}b^2\cdot(T-t-1)}{650d^{3}}\right\}+
        \frac{\min\{m,d\}}{T^3} ,
     \end{align*}
where the two parts in the third line come from the Lemma \ref{lemma:random_resource_stability} and Lemma \ref{lemma:tilde_tau_control}, respectively.
Then, if we denote $u \wedge v = \min \{u,v\}$, we have
\begin{align*}
    \E[T-\tau'] \le & \sum_{t=1}^T  \mathbb{P}\left( \tau'\leq t\right) \\
     \le &\sum_{t=1}^{T} 1 \wedge \left( \frac{\min\{m,d\}}{T^3}+2d\exp\left\{-\frac{\min\{1,\sigma^2\}\min\{\chi^2,\delta^2\}b^2\cdot(T-t-1)}{650d^{3}}\right\}\right)\\
     \le &\frac{650d^3}{b^2\min\{\chi^2,\delta^2\}\min\{1,\sigma^2\}}+\frac{\min\{m,d\}}{T^2}.
\end{align*}
Thus, since $\tau'\leq\tau$ and $\bm{B}^{(\tau')}\leq(T-\tau')\cdot(\bm{b}+\epsilon\bm{1})\leq(T-\tau')\cdot\frac{6}{5}\bm{b}$, we have
\begin{align*}
    \mathbb{E}[\bm{B}^{(\tau)}]
    \leq&
    \mathbb{E}[\bm{B}^{(\tau')}]\\
    =&
    O\left(\frac{d^3}{b\min\{\chi^2,\delta^2\}\min\{1,\sigma^2\}}\right).
\end{align*}
Thus we complete the proof.
\end{proof}

\subsection{Proofs of Lemma \ref{lemma:random_resource_stability} and Lemma \ref{lemma:tilde_tau_control}}

In this subsection, we prove the two key lemmas in the previous section. The proof of Lemma \ref{lemma:random_resource_stability} will utilize the following lemma. The statement of Lemma \ref{adaptLP_sol} is quite straightforward: under the assumption that the parameters are well estimated, that all the arms are optimal, and that all the constraints are binding, the optimal solution to LP \eqref{adaptLP} can be simply obtained by $\left(\bm{C}^{L}(t)\right)^{-1}\bm{b}^{(t)}$. Its proof is solely based on linear algebra knowledge, but it is a bit tedious, for which we will postpone to after the proof of Lemma \ref{lemma:random_resource_stability} and Lemma \ref{lemma:tilde_tau_control}. 

\begin{lemma}
\label{adaptLP_sol}
Under Assumption \ref{warmStart} and Assumption \ref{assume_all_binding}, if $\bm{b}^{(t)}\in\mathcal{Z}$, we have $\hat{\bm{x}}(t)=\left(\bm{C}^{L}(t)\right)^{-1}\bm{B}^{(t)}$ is the optimal solution of LP \eqref{adaptLP} at time $t+1$.
\end{lemma}

\paragraph{Proof of Lemma \ref{lemma:random_resource_stability}}
\begin{proof}
First, we show that each arm will be played, on expectation, no less than $\frac{\chi t}{5}$ times in the first $t$ time periods. Denote the normalized probability used to play arm at time $t+1$ as $\tilde{\bm{x}}(t),$ then by the definition of $\bm{b}^{(t)}$, we know $$\tilde{\bm{x}}(t) = \left({\bm{C}}^{L}(t)\right)^{-1}\bm{b}^{(t)}.$$
Here, from Assumption \ref{assume_all_binding}, we know that $m=d$ in this case, so the constraint consumption matrix ${\bm{C}}^{L}(t)$ is a squared matrix. By the Lemma \ref{adaptLP_sol} and the definition of $\mathcal{Z},$ if $\bm{b}^{(t)}\in \mathcal{Z},$
  \begin{align}
        \bm{C}^{-1}\bm{b}^{(t)}
        &\geq
        \bm{C}^{-1}{\bm{b}}
        -\left\| \bm{C}^{-1}\left(\bm{b}-\bm{b}^{(t)}\right)\right\|_{\infty} \nonumber \\
        &\geq
        \chi-\| \bm{C}^{-1}\|_{\infty}\left\|\left(\bm{b}-\bm{b}^{(t)}\right)\right\|_{\infty}\nonumber \\
        &\geq
        \chi-\frac{\sqrt{d}}{\sigma}\frac{\chi}{5d^{3/2}}
        \geq
        \frac{4\chi}{5}, \label{x_t_part1}
   \end{align}  
where the first line comes from the triangle inequality, the second line comes from the definition of $\chi$, and the third line is obtained by the relation between the singular value and matrix infinity norm. Now we can transfer the bound to $\tilde{\bm{x}}(t).$ From warm start condition \eqref{warm_start} that
$$
\|{\bm{c}}^{L}_{i}(t)-\bm{c}_i\|_{\infty} \leq\frac{\min\{1,\sigma^2\}\min\{\chi,\delta\}}{12\min\{md,d^2\}},\ i\in\mathcal{I}^*,
$$
with the similar proof in the first part of Lemma \ref{adaptLP_sol}, we can show
\begin{align}
    &\left\|\left({\bm{C}}^{L}(t)\right)^{-1}\bm{b}^{(t)}-\bm{C}^{-1}\bm{b}^{(t)}\right\|_{\infty}
    \leq \frac{3\chi}{5}.
    \label{x_t_part2}
\end{align}

Specifically, the warm start condition \eqref{warm_start} provides an upper bound for deviation of ${\bm{C}}^{L}(t)$ and $\bm{C}.$
Putting together \eqref{x_t_part1} and \eqref{x_t_part2}, we have 
    \begin{align*}
       \tilde{\bm{x}}(t) = &\left({\bm{C}}^{L}(t)\right)^{-1}\bm{b}^{(t)}\\
        \geq&
        \bm{C}^{-1}\bm{b}^{(t)}
        -
        \left\|\left({\bm{C}}^{L}(t)\right)^{-1}\bm{b}^{(t)}
        -\bm{C}^{-1}\bm{b}^{(t)}\right\|_{\infty}\\
        \geq&
        \frac{\chi}{5}
 \end{align*} 
where the inequality holds element-wise. This inequality provides an lower bound on the probability distribution with which the arms are played. The implication is that if $\bm{b}^{(t)}\in \mathcal{Z},$ each of the arms will be played with probability no less than $\frac{\chi}{5}.$ In this way, as long as $\bm{b}^{(t)}\in \mathcal{Z}$ always holds, all the arms will be played with sufficiently amount of time to refine the corresponding estimation.
    
Now, we apply a concentration argument to show that each arm will be played for sufficiently many times, not only on expectation, but also with high probability. Concretely, we show that if $\bm{b}^{(s)}\in \mathcal{Z}$ holds for all $s\in[t]$ and $t\ge \frac{200\log T}{\chi^2},$ then the number of times that the $i$-th arm is played up to time $t$, $$n_i(t)\ge \frac{\chi t}{10}$$ with high probability for each arm $i\in [m]$. To see this, we first notice that if $\bm{b}^{(s)}\in \mathcal{Z}$ for all $s\in[t]$, then
$$\sum_{s=1}^t\tilde{x}_i(s) \ge \frac{\chi t}{5}$$
for each $i\in [m].$
Also, note that $n_i(t)-n_i(t-1)-\tilde{x}_i(t)$ is a martingale difference with respect to the filtration $\mathcal{H}_t$, we have the following holds for all $t\in[T],$
\begin{align*}
        &\mathbb{P}\left(
            \sum\limits_{s= 1}^{t}n_i(s)-n_i(s-1)-\tilde{x}_i(s)
            \leq
            - \sqrt{2t\log T}
        \right) \le 
        \frac{1}{T^4},
\end{align*}
by an application of Azuma–Hoeffding inequality. By rearranging terms, it becomes
\begin{align*}
        &\mathbb{P}\left(
          n_i(t)
            \geq
            \sum\limits_{s= 1}^{t}\tilde{x}_i(s)
            - \sqrt{2t\log T}
        \right) \le 
        \frac{1}{T^4},
\end{align*}
and consequently,
\begin{align}
        &\mathbb{P}\left(
          n_i(t)
            \geq
            \sum\limits_{s= 1}^{t}\tilde{x}_i(s)
            - \sqrt{2t\log T}, \text{ for some } t\in[T]
        \right) \le 
        \frac{1}{T^3}.
        \label{n_i_t_ineq}
\end{align}
On the event $$ \left\{n_i(t)
            \geq
            \sum\limits_{s= 1}^{t}\tilde{x}_i(s)
            - \sqrt{2t\log T},
            \text{ for some $t\in[T]$}
            \right\}^{\complement},$$
we know that the inequality $n_i(t)\geq\frac{\chi t}{10}$ holds if $\bm{b}^{(s)}\in\mathcal{Z}$ for all $s\in [t]$ and $t\ge \frac{200\log T}{\chi^2}.$

In other words, when $\bm{b}^{(s)}\in\mathcal{Z}$ for all $s\in [t]$, the $i$-th arm will be played for $\frac{\chi t}{10}$ with probability $1-\frac{1}{T^3}$, which implies the estimation error of the $i$-th arm is at most $\sqrt{\frac{20 \log T}{\chi t}}$ at the beginning of time $t+1$ with high probability. Then the bias in resource consumption has the following upper bound,
    \begin{align*}
        \left\|\mathbb{E}\left[\bm{C}_{i_{t+1}}(\bm{b}')|\mathcal{H}_{t}\right]-\bm{b}'\right\|_{\infty}
        =
        &\left\|\bm{C}\left(\bm{C}^{L}(t)\right)^{-1}\bm{b}'-\bm{b}'\right\|_{\infty}\\
        \leq
        & \left\|\left(\bm{C}-\bm{C}^{L}(t)\right)\left(\bm{C}^{L}(t)\right)^{-1}\bm{b}'\right\|_{\infty}\\
        \leq
        &\left\|\bm{C}-\bm{C}^{L}(t)\right\|_{\infty}\left\|\left(\bm{C}^{L}(t)\right)^{-1}\right\|_{\infty}\|\bm{b}'\|_{\infty}\\
        \leq
        &\sqrt{\frac{{20\min\{md,d^2\}\log T}}{\chi t}}\cdot\frac{2\sqrt{d}}{\sigma}\cdot\frac{6b}{5}\\
        \leq
        &\frac{11b\sqrt{\min\{md^2,d^3\}\log T}}{\sqrt{\chi\sigma^2(t+1)}} 
    \end{align*}
where the last line is defined to be $\epsilon_t$ when $t\ge \alpha T.$ Here,  the second line comes from the definition of the algorithm, the third and fourth line come from the property of matrix multiplication, and the fifth line comes from plugging in the estimation error of $\bm{C}^{L}(t)$. The choice of $\alpha$ ensures that $t\ge \alpha T \ge \frac{200 \log T}{\chi^2}$. Also, we choose $\underline{T}$ as the smallest integer such that
$$\alpha T \ge \frac{200 \log T}{\chi^2}.$$
In this way, when $T\ge \underline{T},$ we have the following two cases:
\begin{itemize}
    \item $t > \alpha T:$ We have
    $$\mathcal{E}_t
    \subset
            \left\{n_i(t)
            \geq
            \sum\limits_{s= 1}^{t}\tilde{x}_i(s)
            - \sqrt{2t\log T},
            \text{ for all $t$}
            \right\}^{\complement}, 
    $$ 
    \item $t\le \alpha T$: We have  $\bm{0}\leq\mathbb{E}\left[\bm{c}_{i_{t}}(\bm{b}')|\mathcal{H}_{t}\right]\leq \bm{1}$ and $\bm{0}\leq \bm{b}'\leq \frac{6}{5}\bm{1}$. Consequently, the following holds 
    $$\left\|\mathbb{E}\left[\bm{c}_{i_{t}}(\bm{b}')|\mathcal{H}_{t-1}\right]-\bm{b}'\right\|_{\infty}\leq\frac{6}{5}$$
    for all $\bm{b}'\in \mathcal{Z}.$ Thus it implies that $\mathcal{E}_t = \emptyset.$
\end{itemize}
Combining the two cases, when $T\ge \underline{T},$ we have
$$\bigcup\limits_{t=1}^{T}\mathcal{E}_t
        \subset
            \left\{n_i(t)
            \geq
            \sum\limits_{s= 1}^{t}\tilde{x}_i(s)
            - \sqrt{2t\log T},
            \text{ for all $t\in [T]$}
            \right\}^{\complement}
 $$
 and therefore, we know from \eqref{n_i_t_ineq},
    $$ \mathcal{P}
    \left(
        \bigcup\limits_{t=1}^{T}\mathcal{E}_t
    \right)
    \leq
    \frac{\min\{m,d\}}{T^3}.  $$
\end{proof}

\paragraph{Proof of Lemma \ref{lemma:tilde_tau_control}}

\begin{proof}
Note that
\begin{align*}
        \left\{{b}^{(s)}_j\not\in[b-\epsilon,b+\epsilon] \text{for some $s\leq t$, }
            \bigcap\limits_{s=1}^{t}\mathcal{E}_t^{\complement}\right\}
        \subset&
        \left\{
            \tilde{{b}}^{(s)}_j\not\in[b-\epsilon,b+\epsilon] \text{ for some } s\leq t
        \right\}
\end{align*}
for each constraint $j\in[d].$
Thus, we only need to analyze the event 
$$
        \left\{
            \tilde{{b}}^{(s)}_j\not\in[b-\epsilon,b+\epsilon] \text{ for some } s\leq t
        \right\}.
$$ 
Note that from telescoping sum,
\begin{align*}
 \tilde{{b}}^{(t)}_j-b & =  \tilde{{b}}^{(t)}_j-\tilde{{b}}^{(0)}_j \\
 & = \sum_{s=1}^t \left( \tilde{{b}}^{(s)}_j-\tilde{{b}}^{(s-1)}_j\right).
\end{align*}
We define $\beta_t\coloneqq\tilde{{b}}_{j}^{(t)}-\tilde{{{b}}}_{j}^{(t-1)}$, and then have
\begin{align}
 \tilde{{b}}^{(t)}_j-b & = \sum_{s=1}^t \beta_t \nonumber \\
 & = \sum_{s=1}^t \left(\beta_s-\mathbb{E}[\beta_s|\mathcal{H}_{s-1}]\right) + \sum_{s=1}^t \mathbb{E}[\beta_s|\mathcal{H}_{s-1}].
 \label{beta_decompose}
\end{align}
We remark that we do not index the process $\beta_t$ in that the analyses for all the constraints are the same, so we only focus on the analysis of a specific constraint $j$. Moreover, from the definition of $\tilde{\tau}$,
    \begin{align*}
        \beta_t
        &=\tilde{{b}}_{j}^{(t)}-\tilde{{b}}_{j}^{(t-1)}
        =
        \left\{
            \begin{matrix}
                0, & t>\tilde{\tau},\\ -\frac{1}{T-t}(C_{j,t}-\tilde{{b}}_{j}^{(t-1)}),   & t\leq\tilde{\tau}.
            \end{matrix}
        \right.
    \end{align*}

Next, we first develop a concentration argument for the first summation in \eqref{beta_decompose}. Specifically, the summand can be viewed as a martingale difference sequence. Note that $$|\beta_t-\mathbb{E}[\beta_t|\mathcal{H}_{t-1}]|\leq\frac{1+\epsilon}{T-t}$$
for $t\in[T-1].$ By applying Hoeffding's Inequality, we have 
\begin{align*}
\mathbb{P}\left(\left\vert\sum_{k=1}^s(\beta_k-\mathbb{E}\left[\beta_k|\mathcal{H}_{k-1}\right]) \right\vert \ge l \text{ for some } s\le t\right)
\leq& 2\exp\left\{-\frac{2l^2}{4\sum\limits_{s=1}^{t}\frac{(1+\epsilon)^2}{(T-t)^{2}}}\right\} \\
=&2\exp\left\{-\frac{(T-t-1)l^2}{2(1+\epsilon)^2}\right\}
\end{align*}
holds for all $t\leq T-2$ and $l \geq 0$. By taking $l=\frac{\epsilon}{3}$,  we have
    \begin{align}
        &\mathbb{P}\left(\left\vert\sum_{k=1}^s \left(\beta_k-\mathbb{E}[\beta_k|\mathcal{H}_{k-1}]\right) \right\vert \ge \frac{\epsilon}{3} \text{ for some } s\le t\right)
        \leq
        2\exp\left\{\frac{-(T-t+1)\epsilon^2}{18(1+\epsilon)^2}\right\}. \label{beta_part1}
    \end{align}

Then we analyze the second summation in \eqref{beta_decompose}. For $s\in[t]$,
    \begin{align*}
        \left\vert
            \sum\limits_{k=1}^{s} \mathbb{E}[\beta_k|\mathcal{H}_{k-1}]
        \right\vert
        &=
        \left\vert
            \sum\limits_{k=1}^{\min\{s,\tilde{\tau}\}} \mathbb{E}[\beta_k|\mathcal{H}_{k-1}]
        \right\vert        \\
        &\leq
        \sum\limits_{k=1}^{\tilde{\tau}}\left\vert
             \mathbb{E}[\beta_k|\mathcal{H}_{k-1}]
        \right\vert        \\
        &\leq
        \sum\limits_{k=1}^{\alpha T}\frac{\bar{\epsilon}}{T-k}
        +
        \sum\limits_{k=\alpha T+1}^{T-1}\frac{\epsilon_t}{T-k}\\
        &\leq
        \frac{\alpha\bar{\epsilon}}{1-\alpha}
        +
        \sum\limits_{k=\alpha T+1}^{T-1}\frac{\epsilon_t}{T-k},
    \end{align*}    
where the last line correspond to the condition of \eqref{eps_condition}. In above, the first line comes from the definition of $\tilde{\tau}$ and the third line comes from the definition of $\mathcal{E}_t$. 

So, if \eqref{eps_condition} holds, i.e.,
$$\frac{\alpha\bar{\epsilon}}{1-\alpha}
        +
        \sum\limits_{l=\alpha T+1}^{T-1}\frac{\epsilon_t}{T-l}\leq\frac{2}{3}\epsilon,$$ we have 
    \begin{equation}
               \left\vert
            \sum\limits_{k=1}^{s} \mathbb{E}[\beta_k|\mathcal{H}_{k-1}]
        \right\vert
        \leq
        \frac{2\epsilon}{3}, \label{beta_part2} 
    \end{equation}
for all $s\in[t]$. 

Finally, by combining \eqref{beta_part1} and \eqref{beta_part2} together into \eqref{beta_decompose}, we have
    \begin{align*}
        &\left\{
            \tilde{{b}}_{j}^{(s)}\not\in\left[b-\epsilon,b+\epsilon\right]\text{ for some }s\leq t
        \right\}\\
        \subset&
        \left\{
            \left|\sum\limits_{k=1}^{s} \beta_k\right|\geq\epsilon \text{ for some }s\leq t
        \right\}\\
        \subset&
        \left\{
            \left\vert\sum_{k=1}^s \left(\beta_k-\mathbb{E}[\beta_k|\mathcal{H}_{k-1}]\right) \right\vert \ge \frac{\epsilon}{3} \text{ for some } s\le t
        \right\}\bigcup\left\{
            \left\vert
                \sum\limits_{k=1}^{s} \mathbb{E}[\beta_k|\mathcal{H}_{k-1}]
            \right\vert
            >
            \frac{2\epsilon}{3}
             \text{ for some } s\le t
        \right\}.     
    \end{align*}
Therefore,
    \begin{align*}
        \mathbb{P}\left(
            \tilde{{b}}_{j}^{(s)}\not\in\left[b-\epsilon,b+\epsilon\right]\text{ for some }s\leq t
        \right)
        \leq&
        \mathbb{P}\left(
            \left\vert\sum_{k=1}^s \left(\beta_k-\mathbb{E}[\beta_k|\mathcal{H}_{k-1}]\right) \right\vert \ge \frac{\epsilon}{3} \text{ for some } s\le t
        \right)\\
        \leq&
        2\exp\left\{\frac{-(T-t+1)\epsilon^2}{18(1+\epsilon)^2}\right\}.
    \end{align*}
As noted earlier, the proof is applicable to all the constraint $j$. And by taking an union bound, we complete the proof.  
\end{proof}

\paragraph{Proof of Lemma \ref{adaptLP_sol}}

\begin{proof}
Consider the following LPs for all $i\in\mathcal{I}^*$, and denote their optimal objective value as $\tilde{\text{OPT}}^{U}_i(t)$ and $\tilde{\text{OPT}}^{U}_{\text{LP}}(t)$, respectively.
    \begin{align*}
        \tilde{\text{OPT}}_i^{U}(t)=\max_{\bm{x}} \ \ & \left(\bm{\mu}^U(t) \right)^\top \bm{x},\\
        \text{s.t.}\ \ &  \bm{C}^L(t) \bm{x} \le \bm{B}^{(t)},  \nonumber  \\
        & x_i=0,\ \bm{x}\ge \bm{0}. \nonumber
    \end{align*}
    \begin{align}
    \label{adapt_LCBLP}
        \tilde{\text{OPT}}^{U}_{\text{LP}}(t)=\max_{\bm{x}} \ \ & \left(\bm{\mu}^U(t) \right)^\top \bm{x},\\
        \text{s.t.}\ \ &  \bm{C}^L(t) \bm{x} \le \bm{B}^{(t)},  \nonumber  \\
        & \bm{x}\ge \bm{0}. \nonumber
    \end{align}
    
Then, similar as the proof in Lemma \ref{OPTij}, to prove that $\hat{\bm{x}}(t)$ is a optimal solution to LP \eqref{adaptLP}, it is sufficient to show that $\hat{\bm{x}}(t)$ is feasible to LP \eqref{adaptLP} and  $\tilde{\text{OPT}}_i^{U}(t)<\tilde{\text{OPT}}^{U}_{\text{LP}}(t)$ for all $i\in\mathcal{I}^*$. 
    
First, we show the feasibility. To see its feasibility, it is sufficient to check the following two results:
\begin{itemize}
    \item[(a)] The matrix ${\bm{C}}^{L}(t)$ is non-singular and thus $\hat{\bm{x}}(t)$ is well-defined.
    \item[(b)] $\hat{\bm{x}}(t)\ge \bm{0}$ and thus $\hat{\bm{x}}(t)$ is a feasible solution to the LP \eqref{adaptLP}.
\end{itemize}

To show part (a), we prove that the smallest singular value of the matrix is positive. Recall we use $\sigma_{\min}(\bm{M})$ and $\sigma_{\max}(\bm{M})$ to denote the smallest and the largest singular value of a matrix $\bm{M}$, respectively.
\begin{align*}
   \sigma_{\min}\left({\bm{C}}^{L}(t)\right)
    &\geq
    \sigma_{\min}\left({\bm{C}}\right) -\sigma_{\max}\left({\bm{C}}^{L}(t)-\bm{C}\right)\\
    &=
    \sigma - \sigma_{\max}\left({\bm{C}}^{L}(t)-\bm{C}\right)\\
    &\geq
    \sigma - \sqrt{d}\|{\bm{C}}^{L}(t)-\bm{C}\|_{\infty}\\
    &\geq
    \sigma - \sqrt{d\cdot\min\{m,d\}}\max\limits_{i\in\mathcal{I}^*}\|{\bm{c}}^{L}_i-\bm{c}_i\|_{\infty},
\end{align*}
where the first line comes from Weyl's inequality on matrix eigenvalues/singular values, the second line comes from the definition of $\sigma$, the third line is obtained from the relation between the spectral norm and the infinity norm of a matrix, and the last line comes from the relation between the matrix infinity norm and the vector infinity norm of each column in the matrix. From the warm start condition (\ref{warm_start}), we have
$$\sigma_{\min}\left({\bm{C}}^{L}(t)\right) \geq \frac{\sigma}{2},$$ and consequently,
\begin{align}
\label{bdd_purturb_matrix}
\sigma_{\max}\left(\left({\bm{C}}^{L}(t)\right)^{-1}\right)\leq\frac{2}{\sigma}.
\end{align}

For part (b), we show the non-negativeness of the solution $\hat{\bm{x}}(t).$ The starting point is the condition that the optimal solution to the Primal LP (\ref{primalLP}) has strictly positive elements, i.e., from Section \ref{param_def}, we know
$${\bm{x}}^* \geq \chi T>0,$$
where the inequality holds element-wise. Now, we show that, 
$$
\left\|\hat{\bm{x}}(t)-{\bm{x}}^*\right\|_{\infty}
=
\left\|\left({\bm{C}}^{L}(t)\right)^{-1}\bm{B}-\bm{C}^{-1}\bm{B}\right\|_{\infty}
\leq \frac{\chi T}{2},
$$
which implies that  
$$
    \hat{\bm{x}}(t)\geq\frac{\chi T}{2}
$$
holds element-wise by triangle inequality. To see this, from condition (\ref{warm_start}) that 
$\max\limits_{i\in\mathcal{I}^*}\|{\bm{c}}^{L}_i-\bm{c}_i\|_{\infty}\leq\frac{\chi\sigma^2}{12{d\min\{m,d\}}}$, then
\begin{align*}
    \left\|\left({\bm{C}}^{L}(t)\right)^{-1}\bm{B}-\bm{C}^{-1}\bm{B}\right\|_{\infty}
    &\leq
    \left\|\left({\bm{C}}^{L}(t)\right)^{-1}-\bm{C}^{-1}\right\|_{\infty}\|\bm{B}\|_{\infty}\\
    &\leq
    T\left\|\left({\bm{C}}^{L}(t)\right)^{-1}-\bm{C}^{-1}\right\|_{\infty}\\
    &\leq
    T\|\bm{C}^{-1}\|_{\infty} \left\|{\bm{C}}\left({\bm{C}}^{L}(t)\right)^{-1}-\bm{I}\right\|_{\infty}\\
    &\leq
    \frac{\sqrt{d}T}{\sigma}\left\|\left({\bm{C}}-{\bm{C}}^{L}(t)+{\bm{C}}^{L}(t)\right)\left({\bm{C}}^{L}(t)\right)^{-1}-\bm{I}\right\|_{\infty}\\
    &=
    \frac{\sqrt{d}T}{\sigma}\left\|\left({\bm{C}}-{\bm{C}}^{L}(t)\right)\left({\bm{C}}^{L}(t)\right)^{-1}\right\|_{\infty}\\
    &\leq
    \frac{\sqrt{d}T}{\sigma}\left\|{\bm{C}}-{\bm{C}}^{L}(t)\right\|_{\infty}\left\|\left({\bm{C}}^{L}(t)\right)^{-1}\right\|_{\infty}\\
    &\leq 
    \frac{2{d}T}{\sigma^2}\left\|{\bm{C}}-{\bm{C}}^{L}(t)\right\|_{\infty}\\
    &\leq
     \frac{2\min\{md,d^2\}T}{\sigma^2}\max_{i\in \mathcal{I}^*}\|{\bm{c}}^{L}_i-\bm{c}_i\|_\infty \le \frac{\chi T}{2}.
\end{align*}
The first line comes from the definition of matrix L$_\infty$ norm. The second line comes from the assumption that $B\le T$. The third and sixth line come from the sub-multiplicativity of matrix L$_\infty$ norm. The fourth and seventh line come from the definition of $\sigma$ following Section 3.3 and the relation between the spectral norm $\sigma_{max}$ and L$_\infty$ norm, i.e., $\left\|\left({\bm{C}}^{L}(t)\right)^{-1}\right\|_{\infty}\leq \sqrt{d}\ \sigma_{max}\left(\left({\bm{C}}^{L}(t)\right)^{-1}\right).$ The last line reduces the matrix infinity norm to vector infinity norm and applies condition (\ref{warm_start}). Thus we finish the part on the feasibility of $\hat{\bm{x}}(t).$
    
Next, we show $\tilde{\text{OPT}}_i^{U}(t)<\tilde{\text{OPT}}^{U}_{\text{LP}}(t)$ for all $i\in\mathcal{I}^*$ by utilizing $\text{OPT}_i^{U}(t)$ and $\text{OPT}^{U}(t)$. 
    
    Recall that $\text{OPT}_i^{U}(t)$ is the optimal objective value of the following LP. 
    \begin{align*}
        \text{OPT}_{i}^U(t) \coloneqq  \max_{\bm{x}} \ \ & \left(\bm{\mu}^U(t-1)\right)^\top \bm{x}, \\
        \text{s.t.}\ \ &  \bm{C}^L(t-1) \bm{x} \le \bm{B},  \nonumber  \\
        & x_i=0, \bm{x}\ge \bm{0}. \nonumber
    \end{align*} 
    With the similar proof in Lemma \ref{lem:conv} and Assumption \ref{warmStart}, we can show that 
    \begin{align}
    \label{warm_conv}
        \text{OPT}^{U}_i(t)&\leq\text{OPT}_i+\frac{\delta}{5}T.
    \end{align}
    
Then, from LP's duality, we know that $\text{OPT}^{U}_i(t)$ and $\tilde{\text{OPT}}^{U}_i(t)$ are also the optimal values of the dual LP \eqref{LCB_DLP} and \eqref{adapt_LCBDLP}, respectively.
    \begin{align}
    \label{LCB_DLP}
        \min_{\bm{y}} \ \ & \bm{B}^{\top}\bm{y} \nonumber,\\
        \text{s.t.}\ \ & ({{\bm{C}}}_{-i}^{L}(t-1))^{\top}\bm{y} \geq {\bm{\mu}}_{-i}^{U}(t-1),\\
        &\bm{y}\ge 0. \nonumber
    \end{align} 
    \begin{align}
    \label{adapt_LCBDLP}
        \min_{\bm{y}} \ \ & \left(\bm{B}^{(t-1)}\right)^{\top}\bm{y} \nonumber,\\
        \text{s.t.}\ \ & ({{\bm{C}}}_{-i}^{L}(t-1))^{\top}\bm{y} \geq {\bm{\mu}}_{-i}^{U}(t-1),\\
        &\bm{y}\ge 0, \nonumber
    \end{align} 
    where $\bm{M}_{-i}$ is a matrix (vector) that delete the $i$-th column (entry) of the matrix (vector) $\bm{M}$. Denote the optimal solution of LP \eqref{LCB_DLP} as $\tilde{\bm{y}}$. By duality, we have 
    \begin{align*}
        \tilde{y}_j
        =\frac{1}{B}B\tilde{y}_j
        \leq\frac{1}{B}\bm{B}^{\top}\tilde{\bm{y}}
        \leq \frac{1}{b}.
    \end{align*}
    for all $j\in[d]$. Moreover, since $\tilde{\bm{y}}$ is also feasible to LP \eqref{adapt_LCBDLP}, by duality, we have
    \begin{align*}
        \tilde{\text{OPT}}_i^{U}(t)
        \leq&
        (\bm{B}^{(t-1)})^\top\tilde{\bm{y}}\\
        =&
        \frac{T-t}{T}(\bm{B})^\top\tilde{\bm{y}}+(\bm{B}^{(t)}-(T-t)\bm{b})^\top\tilde{\bm{y}}\\
        \leq&
        \frac{T-t}{T}\text{OPT}_{i}^{U}(t)+\frac{d}{b}\cdot\frac{\delta}{5d^{3/2}}(T-t)\\
        \leq&
        \frac{T-t}{T}\text{OPT}_i^{U}(t)+\frac{\delta}{5}(T-t)\\
        \leq&
        \frac{T-t}{T}\text{OPT}_i+\frac{2\delta}{5}(T-t).
    \end{align*}
    where the first line comes from weak duality, the second line comes from re-arranging terms, the third line comes from  $\|\bm{b}^{(t)}-\bm{b}\|_{\infty}\leq\frac{\delta}{5d^{3/2}}$ and Holder's Inequality that $|(\bm{B}^{(t)}-(T-t)\bm{b})^\top\tilde{\bm{y}}|\leq\|\bm{B}^{(t)}-(T-t)\bm{b}\|_{\infty}\|\tilde{\bm{y}}\|_1$, and the last line is obtained by the inequality \eqref{warm_conv}. Then, we can have a similar statement that 
    $$
        \tilde{\text{OPT}}^{U}_{\text{LP}}(t)\geq
        \frac{T-t}{T}{\text{OPT}}_{\text{LP}}-\frac{2\delta}{5}(T-t).
    $$
    Thus, since $\text{OPT}_i\leq\text{OPT}_{\text{LP}}-\delta$, we have 
    \begin{align*}
        \tilde{\text{OPT}}^{U}_{\text{LP}}(t)\geq&
        \frac{T-t}{T}{\text{OPT}}_{\text{LP}}-\frac{2\delta}{5}(T-t)\\
        \geq&
        \frac{T-t}{T}\tilde{\text{OPT}}^{U}_i(t)+\frac{\delta}{5}(T-t)\\
        >&
        \tilde{\text{OPT}}^{U}_i(t),
    \end{align*}
for all $i\in\mathcal{I}^*$.  Thus we complete the proof by combining the feasibility and the optimality of $\hat{\bm{x}}(t).$
\end{proof}

\subsection{Relaxation of Warm Start Assumption}
\label{sectionWarmStart}

In this subsection, we relax Assumption \ref{warmStart} for the analysis in Section \ref{MainProof}. Indeed, The following proposition shows that the warm start condition in Assumption \ref{warmStart} will be automatically satisfied after $\frac{400\log T}{\chi\theta^3}$ time periods after Phase I, where $\theta$ being the threshold parameter in Assumption \ref{warmStart}. The idea of the proof is straightforward: at the beginning of Phase II, although the parameter estimation may not satisfy the warm start condition, it will ensure that each arm will be played with at least certain probability. Then, by continually playing the arms, the corresponding parameters for the optimal arms will be gradually refined until the warm start condition is met. The refinement takes at most $\frac{400\log T}{\chi\theta^3}$ time periods with high probability.

\begin{proposition}
In Phase II of Algorithm \ref{alg_BwK}, it will automatically achieve the warm start condition in Assumption \ref{warmStart} after $\frac{400\log T}{\chi\theta^3}$ time periods (excluding the periods in Phase I) with probability $1-\frac{20md}{T}$.
\end{proposition}

\begin{proof}
First, we show that each arm will be played linear times even if the Phase II does not initialize with a warm start, or, sufficiently, there exists an arm $i\in\mathcal{I}^*$ so that $2\sqrt{\frac{2\log T}{n_i(t)}}>\theta$ for some $t$ in Phase II. W.L.O.G., we assume that all UCB estimates is non-increasing and all LCB estimates are non-decreasing since we can use the minimal UCB estimates and maximal LCB estimates for all quantities from time 1 to $t-1$ at each time $t$. Thus, $\text{OPT}^{U}_{\text{LP}}(t)$ and $\text{OPT}^{U}_i$ are non-increasing and $\text{OPT}^{L}_{\text{LP}}$ is non-decreasing. 
    
Now, consider the following LPs for all $i\in\mathcal{I}^*$ (the same as the LPs used in Lemma \ref{adaptLP_sol}):
    \begin{align*}
        \tilde{\text{OPT}}_i^{U}(t)=\max_{\bm{x}} \ \ & \left(\bm{\mu}^U(t-1) \right)^\top \bm{x},\\
        \text{s.t.}\ \ &  \bm{C}^L(t-1) \bm{x} \le \bm{B}^{(t-1)},  \nonumber  \\
        & x_i=0,\ \bm{x}\ge \bm{0}, \nonumber
    \end{align*}
    \begin{align*}
        \tilde{\text{OPT}}^{U}_{\text{LP}}(t)=\max_{\bm{x}} \ \ & \left(\bm{\mu}^U(t-1) \right)^\top \bm{x},\\
        \text{s.t.}\ \ &  \bm{C}^L(t-1) \bm{x} \le \bm{B}^{(t-1)},  \nonumber  \\
        & \bm{x}\ge \bm{0}. \nonumber
    \end{align*}
    Note $\|\frac{\bm{B}^{(t)}}{T-t}-\bm{b}\|_{\infty}\leq\frac{2t}{T}$ if $t\leq\frac{T}{2}$. With the similar proof in Lemma \ref{adaptLP_sol}, we can show that 
    \begin{align*}
        \frac{1}{T-t}|\tilde{\text{OPT}}_i^{U}(t)-\frac{T-t}{T}\text{OPT}_i^{U}(t)|\leq\frac{\min\{\chi\theta,\delta\theta\}}{45},\\
        \frac{1}{T-t}|\tilde{\text{OPT}}^{U}_{\text{LP}}(t)-\frac{T-t}{T}\text{OPT}_{\text{LP}}^{U}(t)|\leq\frac{\min\{\chi\theta,\delta\theta\}}{45}
    \end{align*}
    for all $t\leq\frac{b\min\{\chi\theta,\delta\theta\}}{90d}T$. Then, by triangle inequality, we have, for all $t$ in Phase II,
    \begin{align}
        \label{opti_lp}
        \tilde{\text{OPT}}^{U}_{\text{LP}}(t)
        -
        \tilde{\text{OPT}}_i^{U}(t)
        \geq&
        \frac{T-t}{T}\text{OPT}_{\text{LP}}^{U}(t)-\frac{T-t}{T}\text{OPT}_i^{U}(t)
        -\frac{2\min\{\chi\theta,\delta\theta\}}{45}(T-t)\\
        \geq&
        \frac{T-t}{T}\text{OPT}_{\text{LP}}^{U}(t)-\frac{T-t}{T}\text{OPT}_{\text{LP}}^{L}(t)\nonumber
        -\frac{2\min\{\chi\theta,\delta\theta\}}{45}(T-t),
    \end{align}
    where the second line is obtained by $$\text{OPT}_i^{U}(t)<\text{OPT}^{L}_{\text{LP}}(t)$$ after Phase I. 
    This motivates us to focus on the difference between $\text{OPT}_{\text{LP}}^{U}(t)$ and $\text{OPT}_{\text{LP}}^{L}(t)$.
    With the proof in Lemma \ref{paramEst}, we can show that with probability $1-\frac{4m}{T}$ 
    \begin{align*}
        \mu_i\leq\hat{\mu}_i(t)+\sqrt{\frac{3\log T}{2n_i(t)}},
    \end{align*}
    hold for all $i\in[m]$, $j\in[d]$ and $t\in[T]$. Notice that $\bm{x}^*$ is also a feasible solution to LP \eqref{UCB_LP}. Thus, we have
    \begin{align*}
        \text{OPT}_{\text{LP}}^{U}(t)
        \geq&
        \sum\limits_{i=1}^{m}\mu_i^{U}(t)x_i^*\\
        =&
        \sum\limits_{i=1}^{m}\left(\hat{\mu}_i^{U}(t)+\sqrt{\frac{3\log T}{2n_i(t)}}\right)x_i^*
        +\sum\limits_{i=1}^{m}\left(\sqrt{\frac{2\log T}{n_i(t)}}-\sqrt{\frac{3\log T}{2n_i(t)}}\right)x_i^*\\
        \geq&
        \text{OPT}_{\text{LP}}+\sum\limits_{i=1}^{m}\left(\sqrt{\frac{2\log T}{n_i(t)}}-\sqrt{\frac{3\log T}{2n_i(t)}}\right)\chi T\\
        \geq&
        \text{OPT}_{\text{LP}}+\frac{1}{15}\theta\chi T
    \end{align*}
hold with probability no less than $1-\frac{8md}{T}$ and all $t$ in Phase II. Here, the first line comes from the definition of $\text{OPT}_{\text{LP}}^{U}(t)$, the second line comes from re-arranging terms, the third line comes from $x_i^*\geq\chi$ and $\mu_i\leq\hat{\mu}_i(t)+\sqrt{\frac{3\log T}{2n_i(t)}}$ for all $i$, and the last line comes from calculation. Take this result into the inequality \eqref{opti_lp}, we have
    \begin{align*}
        \tilde{\text{OPT}}^{U}_{\text{LP}}(t)
        -
        \tilde{\text{OPT}}_i^{U}(t)
        \geq
        &\frac{\min\{\chi\theta,\delta\theta\}}{45}(T-t),
    \end{align*}
    for all $i\in\mathcal{I}^*$. Denotes the optimal solution corresponding to $\tilde{\text{OPT}}^{U}_{\text{LP}}(t)$ as $\tilde{\bm{x}}(t)$ and notice that $$\tilde{\text{OPT}}^{U}_{\text{LP}}(t)
        \leq
        \tilde{\text{OPT}}_i^{U}(t)+\tilde{x}_i(t)$$ for all $i\in\mathcal{I}^*$. we have 
    $$
        \tilde{\bm{x}}(t)\geq\frac{\chi\theta}{45}(T-t).
    $$
    Thus, in expectation, the algorithm plays each arm at least $\frac{\chi\theta}{45}$ times for each step if $t\leq\frac{b\min\{\chi\theta,\delta\theta\}T}{90d}$ during Phase II. Then, with similar proof in Lemma \ref{lemma:random_resource_stability}, we can show that there exist a $T_0$ depending polynomially on $m$, $d$, $1/b$, $1/\sigma$, $1/\chi$ and $1/\delta$ such that for all $T>T_0$, the algorithm can play each arm for at least $\frac{8\log T}{\theta^2}$ times within $\frac{400\log T}{\chi\theta^3}$ steps with probability no less than $1-\frac{20md}{T}$. Then, the algorithm achieve the warm start. Moreover, since the algorithm only plays arms in $\mathcal{I}^*$ and $T\geq T_0$ is large enough during this process, it will not cause additional regret.
\end{proof}

\subsection{Analysis of the Non-Binding Constraints}

\label{ana_non_binding}

The way how we analyze the non-binding constraints is the same as how we analyze the binding constraints. Previously, we derive an upper bound for $\E[T-\tau]$ in Section \ref{MainProof} under the assumption that all the constraints are binding. Indeed, with the presence of non-binding constraints, we only need to show that the non-binding constraints will not be exhausted before exhausting any binding constraints. To adapt the analysis in Section \ref{MainProof} to the case when there are non-binding constraints, we only need to change the event $\mathcal{E}_t$ to the joint of following event $\tilde{\mathcal{E}}_t$ and $\mathcal{E}_t$. All the remaining parts in Section \ref{MainProof} can then be extended to the case when there are non-binding constraints,

Define
$$ \tilde{\mathcal{E}}_t
    \coloneqq
    \left\{
    \begin{matrix}
        {b}^{(s)}_j\in [b-\epsilon,b+\epsilon] \text{ for all $s\in [t-1]$} \text{ and for all binding constraints $j\in\mathcal{J}^*$, } \\
        B^{(s)}_j \geq\bm{0} \text{ for all $s\in[t-1]$} \text{ and for all non-binding constraints $j\in\mathcal{J}'$,}\\
        B^{(t)}_j \geq\bm{0} \text{ does not hold} \text{ for some non-binding constraint $j\in\mathcal{J}'$}
    \end{matrix}
    \right\}.
$$
In this way, the event $\tilde{\mathcal{E}}_t$ describes the ``bad'' event that the binding dimensions of the process $\bm{b}^{(s)}$ stays within the region $\mathcal{Z}$ for $s=1,...,t-1$ and the non-binding dimensions of the demand process $\bm{B}^{(s)}$ is not exhausted before time $t-1$, but certain non-binding constraint is exhausted at time $t$. The following lemma establishes that such event will happen with a small probability. It thus implies that as long as $\bm{b}^{(t)}\in\mathcal{Z}$, the non-binding constraints will not be exhausted. Its proof shares the same structure as Lemma \ref{lemma:random_resource_stability}. 
\begin{lemma}
    \label{lemma:non_binding}
        There exists constants $\underline{T}$ depending on $d$, $\chi$ such that for all $T\geq\underline{T}$, the following inequality holds
        $$
            \mathbb{P}\left(
                \bigcup\limits_{t=1}^{T}
                    \tilde{\mathcal{E}}_t
            \right)
            \leq
            \frac{4d}{T^3}.
        $$
\end{lemma}

\paragraph{Proof of Lemma \ref{lemma:non_binding}}
\begin{proof}
Based on the definition of $\delta$, we have  $$\sum\limits_{i=1}^{m}\frac{c_{ij}x_i^*}{T}\leq b-\delta$$
holds for all non-binding constraint $j$, where $\bm{x}=(x_1^*,...,x_m^*)^\top$ denotes the optimal solution to \eqref{primalLP}.  Recall that for non-binding constraints, we only need to ensure that they are not depleted before $\tau'$. The implication of this inequality is that the non-binding constraints have $\delta T$ amount of resource surplus to prevent the depletion.

    Intuitively, if the average consumption of all binding resources are close to $b$, i.e. $b_j^{(s)}\in[b-\epsilon,b+\epsilon]$ for each binding resource $j$, the number of average played times for each optimal arm $i$ will be close to $\frac{\bm{x}_i^*}{T}$. Then, the average consumption of each non-binding resources will close to . Follow the intuition, we will bound the probability of the statement in this lemma by concentration inequality. 
    
    First, we estimate played times for each arm. To bound the average played times, with the similar proof for lemma \ref{lemma:random_resource_stability}, we can show that, with probability no less than $\frac{2d}{T^3}$, the following inequality holds for all $i\in[m]$ and $s\in [T]$:
    $$
        n_i(s)\leq\sum\limits_{l=1}^{s}\tilde{x}_i(l)+\sqrt{2s\log T}.
    $$
    Based on the algorithm, if $b^{(s)}_j\in[b-\epsilon,b+\epsilon]$ for all binding resources, we know that,
    $$
        \tilde{\bm{x}}(s)=({\bm{C}}^{L}(s))_{\mathcal{J}^*,\mathcal{I}^*}^{-1}({\bm{b}}^{(s)})_{\mathcal{J}^*},
    $$
    for all $s\leq\tau_n<\tau_b$. Thus, if $b^{(s)}_j\in[b-\epsilon,b+\epsilon]$, for all $s\leq t$ and binding resources, we have 
    $$\bm{n}(s)\leq\sum\limits_{l=1}^{s}({\bm{C}}^{L}(l))_{\mathcal{J}^*,\mathcal{I}^*}^{-1}({\bm{b}}^{(l)})_{\mathcal{J}^*}+\sqrt{2s\log T}\bm{1},$$
    where $\bm{n}(s)=\{n_i(s)\}_{i=1}^{m}$. Moreover, for all $t\leq T$, on the event $\mathcal{E}_t$, we know that no resource is used up.
     
    Second, we estimate the resource consumption for non-binding resources. Let $\tilde{\mathcal{H}}_s=(\bm{r}_l,\bm{C}_{l},i_{l+1})_{l=1}^{s}$, apply concentration inequality and we have  
    \begin{align*}
        \mathbb{P}\left(
            \sum\limits_{l=1}^{s} \bm{C}_{l}-\mathbb{E}[\bm{C}_{l}|\tilde{\mathcal{H}}_{l-1}]>\sqrt{2s\log T}
        \right)
        \leq&
        2\exp\left\{
            -\frac{4s\log T}{s}
        \right\}\\
        =&
        2\exp\{-4\log T\}\\
        =  &      
        \frac{1}{T^4}.
    \end{align*}
    Thus, with probability no more than $\frac{2d}{T^3}$, we have $\sum\limits_{l=1}^{s} \bm{C}_{l}-\mathbb{E}[\bm{C}_{l}|\tilde{\mathcal{H}}_l]\leq\sqrt{2s\log T}$ for all $s\in[T]$ and $j\in[d]$. Moreover, notice that $\mathbb{E}[\bm{C}_{s}|\tilde{\mathcal{H}}_{s-1}]=\bm{c}_{i_s}$ for all $s\leq t\leq T-2$ on the event $\mathcal{E}_t^{(n)}$,  we have 
    \begin{align*}
        \sum\limits_{s=1}^{t}\mathbb{E}[\bm{C}_{ji_s,s}|\tilde{\mathcal{H}}_l]
        =&
        \bm{C}\bm{n}(t)\\
        \leq
        &\bm{C}\left(\sum\limits_{s=1}^{t}({\bm{C}}^{L}(s))_{\mathcal{J}^*,\mathcal{I}^*}^{-1}({\bm{b}}^{(s)})_{\mathcal{J}^*}+\sqrt{2t\log T}\bm{1}\right) \\
        =
        &\sum\limits_{s=1}^{t}\bm{C}\left( \bm{C}_{\mathcal{J}^*,\mathcal{I}^*}
         \bm{C}^{-1}_{\mathcal{J}^*,\mathcal{I}^*}({\bm{C}}^{L}(s)-\bm{C})_{\mathcal{J}^*,\mathcal{I}^*}\left({\bm{C}}^{L}(s)\right)^{-1}_{\mathcal{J}^*,\mathcal{I}^*}
        \right)({\bm{b}}^{(s)})_{\mathcal{J}^*}\\
        &+
        \sqrt{2t\log T}\bm{C}\bm{1}\\
        \leq
        &\frac{t}{T}\bm{C}\bm{x}^*+\frac{\chi\bm{C}\bm{1}}{5d}t+\sqrt{2t\log T}\bm{C}\bm{1}+\frac{\chi t}{5}\bm{C}\bm{1}\\
        \leq
        &\frac{t}{T}\bm{C}\bm{x}^*+\frac{2\chi t}{5}\bm{1}+\sqrt{2d^2t\log T}\bm{1}
    \end{align*}
    holds on the event 
    \begin{align*}
        \tilde{\mathcal{E}}_t\backslash&\left\{n_i(s)>\sum\limits_{l=1}^{s}\tilde{x}_i(l)+\sqrt{2s\log T} \text{ for some $s\in[T]$ and $i\in[m]$}\right\},    
    \end{align*}
    where the first in the inequalities line comes from $\mathbb{E}[\bm{C}_{i_s,s}|\tilde{\mathcal{H}}_{s-1}]=\bm{c}_{i_s}$, the second line comes from the analysis of $\bm{n}(s)$, the third line comes from the decomposition of $(\hat{\bm{C}}_(s-1))_{\mathcal{J}^*,\mathcal{I}^*}^{-1}$, the fifth line comes from $\bm{b}^{(s)}\in\mathcal{Z}$, and, the last line comes form $\bm{C}\bm{1}\leq d$. Henceforth, on the event
    \begin{align*}
        \tilde{\mathcal{E}}_t
        \backslash
        &\left\{n_i(s)>\sum\limits_{l=1}^{s}\tilde{x}_i(l)+\sqrt{2s\log T}\text{ or } \right.\\
        &\left.\sum\limits_{l=1}^{s} {C}_{ji_l,l}-\mathbb{E}[{C}_{ji_l,l}|\tilde{\mathcal{H}}_l]>\sqrt{2s\log T}\text{ for some $s\in[T]$ $i\in[m]$ and $j\in[d]$}\right\},
    \end{align*}
    and $t\leq T-2$ we have
    \begin{align*}
        \sum\limits_{s=1}^{t}C_{j}(s)
        \leq&
        \frac{t}{T}\left(\bm{C}\bm{x}^*\right)_j+\frac{2\chi t}{5}+2\sqrt{2d^2t\log T}\\
        \leq&
        \left(\bm{C}\bm{x}^*\right)_j+\frac{2\chi t}{5} +2\sqrt{2d^2t\log T}\\
        \leq&
        B-\chi T-\chi+\frac{2\chi T}{5} +2\sqrt{2d^2T\log T}\\
        \leq&
        B-\frac{\chi T}{5},
    \end{align*}
    if $T\geq\frac{50d^2\log T}{\chi^2}$. Thus, as $T\geq\frac{10}{\chi}$, non-binding resources are no less than $2$ at the end of time $t$ for all non-binding resources, implying $\bm{B}^{(t+1)}\geq\bm{1}$. Thus, we have
    \begin{align*}
        \tilde{\mathcal{E}}_t
        \subset
        &\left\{n_i(s)>\sum\limits_{l=1}^{s}\tilde{x}_i(l)+\sqrt{2s\log T}\text{ or } \right.\\
        &\left.\sum\limits_{l=1}^{s} {C}_{ji_l,l}-\mathbb{E}[{C}_{ji_l,l}|\tilde{\mathcal{H}}_l]>\sqrt{2s\log T} \text{ for some $s\in[T]$ $i\in[m]$ and $j\in[d]$}\right\},
    \end{align*}
    and consequently, 
    \begin{align*}
        \bigcup\limits_{t=0}^{T-2}\tilde{\mathcal{E}}_t
        \subset
        &\left\{n_i(s)>\sum\limits_{l=1}^{s}\tilde{x}_i(l)+\sqrt{2s\log T}\text{ or } \right.\\
        &\left.\sum\limits_{l=1}^{s} {C}_{ji_l,l}-\mathbb{E}[{C}_{ji_l,l}|\tilde{\mathcal{H}}_l]>\sqrt{2s\log T}\text{ for some $s\in[T]$ $i\in[m]$ and $j\in[d]$}\right\},
    \end{align*}   
    \begin{align*}
        \mathbb{P}\left(
        \bigcup\limits_{t=0}^{T-2}\tilde{\mathcal{E}}_t
        \right)
        \leq&
        \mathbb{P}
            \left(n_i(s)>\sum\limits_{l=1}^{s}\tilde{x}_i(l)+\sqrt{2s\log T}\text{ for some $s\in[T]$ and $i\in[m]$}
        \right)\\
        &+
        \mathbb{P}
        \left(\sum\limits_{l=1}^{s} {C}_{ji_l,l}-\mathbb{E}[{C}_{ji_l,l}|\tilde{\mathcal{H}}_l]>\sqrt{2s\log T}
        \text{ for some $s\in[T]$ $i\in[m]$ and $j\in[d]$}
        \right)\\
        \leq&
        \frac{4d}{T^3}.
    \end{align*}
\end{proof}

\section{Proof of Proposition \ref{regret_prop}}
\begin{proof}
The regret analysis can be accomplished by a combination of Proposition \ref{prop_upper_bound}, Proposition \ref{PhaseI_UB} and Proposition \ref{prop_B_tau}. Proposition \ref{prop_upper_bound} shows that
    \begin{align}
    \label{regret_up}
            \text{Regret}_T^{\pi}(\mathcal{P}, \bm{B}) \le   \sum_{i\in \mathcal{I}'} n_i(t)\Delta_i +\mathbb{E}\left[\bm{B}^{(\tau)}\right]^\top\bm{y}^*.
    \end{align}
    
First, we bound the first term in the right hand side of inequality \eqref{regret_up}. From the duality, we know that $\bm{B}^{\top}\bm{y}^*\leq T$. Thus, we have
    \begin{align*}
        \bm{y}^*&\leq\frac{1}{b}, \ \ \ 
        \Delta_i\le \bm{1}^{\top}\bm{y}^*\le\frac{d}{b}.
    \end{align*}
    Moreover, Proposition \ref{PhaseI_UB} shows that with probability no less than $1-\frac{4md}{T^2}$, the algorithm only plays arms in $\mathcal{I}'$ during the Phase I and that Phase I will terminate within 
    $$O\left(\left(2+\frac{1}{b}\right)^2\frac{\log T}{\delta^2}\right)$$ time periods. Thus,
    \begin{align}
        \label{term1}
        \sum_{i\in \mathcal{I}'} n_i(t)\Delta_i
        =
        O\left(\left(2+\frac{1}{b}\right)^2\frac{md\log T}{b\delta^2}\right). 
    \end{align}

    Next, for the second term in the inequality \eqref{regret_up}. Applying Proposition \ref{prop_B_tau}, we have 
    \begin{align}
        \label{term2}
        \mathbb{E}\left[\bm{B}^{(\tau)}\right]^\top\bm{y}^*
        \leq&
        \sum\limits_{j\in\mathcal{J}^*}\mathbb{E}\left[B_j^{(\tau)}\right]{y}^*_j\\
        \leq&
        O\left( \frac{d^4}{b^2\min\{\chi^2, \delta^2\}\min\{1,\sigma^2\}}\right), 
    \end{align}
    
Finally, plugging \eqref{term1} and \eqref{term2} into \eqref{regret_up}, we complete the proof.
\end{proof}

\end{document}